\documentclass{article}

\usepackage[preprint]{neurips_2026}

\usepackage[utf8]{inputenc} %
\usepackage[T1]{fontenc}    %
\usepackage{hyperref}       %
\usepackage{url}            %
\usepackage{booktabs}       %
\usepackage{amsfonts}       %
\usepackage{nicefrac}       %
\usepackage{microtype}      %
\usepackage{xcolor}         %

\usepackage[utf8]{inputenc} %
\usepackage[T1]{fontenc}    %
\usepackage{hyperref}       %
\usepackage{url}            %
\usepackage{booktabs}       %
\usepackage{amsfonts}       %
\usepackage{nicefrac}       %
\usepackage{microtype}      %
\usepackage{xcolor}         %

\usepackage{algorithm}
\usepackage{algpseudocode}

\usepackage[utf8]{inputenc} %
\usepackage[T1]{fontenc}    %
\usepackage{hyperref}       %
\usepackage{url}            %
\usepackage{booktabs}       %
\usepackage{amsfonts}       %
\usepackage{nicefrac}       %
\usepackage{microtype}      %
\usepackage{xcolor}         %

\usepackage{bm} %
\usepackage{amsmath}
\usepackage{amssymb}
\usepackage{amsfonts}
\usepackage{graphicx}
\usepackage{amsthm}
\usepackage[shortlabels]{enumitem}
\usepackage{todonotes}
\usepackage[capitalize,noabbrev]{cleveref}
\usepackage{subcaption}

\newcommand{\thetahat}{\hat{\theta}}
\newcommand{\thetastar}{\theta_\star}

\renewcommand{\epsilon}{\varepsilon}
\def\defeq{\stackrel{\mathrm{def}}{=}}
\def\setof#1{\left\{#1  \right\}}
\newcommand{\sep}{\,\vert \,}

\def\abs#1{\left|#1  \right|}

\def\indicator#1{\mathbf{1}\left\{ #1 \right\}}

\theoremstyle{plain}
\newtheorem{theorem}{Theorem}[section]

\newtheorem{lemma}[theorem]{Lemma}
\newtheorem{corollary}[theorem]{Corollary}
\theoremstyle{definition}

\theoremstyle{remark}

\newcommand{\E}{\mathbb{E}}
\newcommand{\Var}{\mathrm{Var}}
\newcommand{\Cov}{\mathrm{Cov}}

\def\calM{\mathcal{M}}
\def\calV{\mathcal{V}}

\def\calQ{\mathcal{Q}}

\newcommand{\R}{\mathbb{R}}

\newcommand\ppi{\boldsymbol{\pi}}

\DeclareMathOperator*{\argmax}{arg\,max}

\usepackage{natbib}
\usepackage{pifont}

\newcommand{\tasktrue}{\theta^{(t)}_\star}
\newcommand{\tasks}{\mathcal{T}}
\newcommand{\task}{t}
\newcommand{\taskindex}{{(t)}}

\newcommand{\dataset}{\mathcal{D}}

\newcommand{\numtasks}{T}
\newcommand{\classical}{\textsc{Classical}\ }
\newcommand{\ppipp}{\textsc{PPI++}\ }
\renewcommand{\ppi}{\textsc{PPI}\ }
\newcommand{\greppi}{\textsc{GRePPI}\ }
\newcommand{\areppi}{\textsc{ARePPI}\ }
\newcommand{\labeled}{\mathcal{L}}
\newcommand{\reppi}{\textsc{RePPI}\ }
\renewcommand{\models}{\calM}
\newcommand{\model}{\mathfrak{m}}
\newcommand{\excepttask}{{(-\task)}}

\newcommand{\gptfive}{\texttt{gpt-5-mini}\ }
\newcommand{\fouromini}{\texttt{gpt-4o-mini}\ }

\newcommand{\xhdr}[1]{\vspace{0mm}\noindent{{\bf #1.}}}

\usepackage{hyperref}
\hypersetup{
  colorlinks = true, 
  urlcolor = {red},
  linkcolor = {blue},
  citecolor = {green!50!black}
}

\newif\ificmlworkshop
\newif\iffullpaper

\title{Prediction-Powered Inference Across Many Tasks for AI Evaluation \& Social Science Research}

\author{
  Nicolas Emmenegger \\
  MIT \\
  Cambridge, MA \\
  \texttt{nemm@mit.edu} 
  \And
  Ellery Stahler \\
  MIT \\
  Cambridge, MA \\
  \texttt{ellerys@mit.edu} 
  \And
  Chara Podimata \\
  MIT\\
  Cambridge, MA \\
  \texttt{podimata@mit.edu} 
}

\icmlworkshopfalse
\fullpapertrue

\begin{document}

\maketitle

\begin{abstract}
  Many applications require statistically valid inference across many related ``\emph{tasks}'', while using only a handful of high-quality labels per hypothesis. In AI evaluation, these tasks may correspond to model behaviors across prompts, subgroups, or hypotheses; in social science surveys, they may correspond to related questions, populations, or measurement conditions.  Prediction-powered inference (PPI) uses abundant but inexpensive proxy measurements to improve inference from limited, ``ground-truth'' labels, but commonly used methods treat tasks independently and therefore fail to exploit shared structure across related tasks. This limitation is especially important in settings where only a small number of labels are available per task. To address this issue, we introduce a multi-task prediction-powered inference framework that uses labeled data from related tasks to improve power while preserving task-specific inference. Our methods exploit the shared structure in the proxy-ground-truth relationship through cross-task recalibration, while retaining within-task rectification and power tuning to construct accurate point estimates and confidence intervals. We prove that efficiency gains beyond power-tuned PPI are only possible when the proxy-ground-truth relationship contains nonlinear structure; affine cross-task recalibrations are asymptotically equivalent to using the original proxy. We complement our theoretical findings with experiments on synthetic and semi-synthetic datasets, as well as a case study auditing language models on election-related information during the 2024 U.S. presidential election. Using a large human-annotation study, we show that cross-task recalibration can substantially reduce confidence interval widths when labels are scarce.
\end{abstract}

\section{Introduction}
\label{sec:intro}

Modern data analysis increasingly requires valid inference across many related \emph{tasks}. In AI evaluation, these tasks may correspond to model behaviors across benchmarks, prompts, population subgroups, or models. In social science surveys, they may correspond to related questions, subject populations, or measurement conditions. In each case, practitioners often have access to inexpensive proxy measurements at scale (such as AI-generated scores or otherwise-automated annotations, or historical responses under different protocols) but can obtain only a small number of high-quality measurements, such as expert annotations, validated human judgements, or survey responses collected under the target protocol. We refer to these high-quality measurements as \emph{ground-truth labels}, not because they are noiseless, but because they define the target measurement process for inference.

\begin{figure*}[t!]
\label{fig:main:humans}
    \centering
    \begin{subfigure}[t]{0.4\textwidth}
        \centering
        \includegraphics[width=\linewidth]{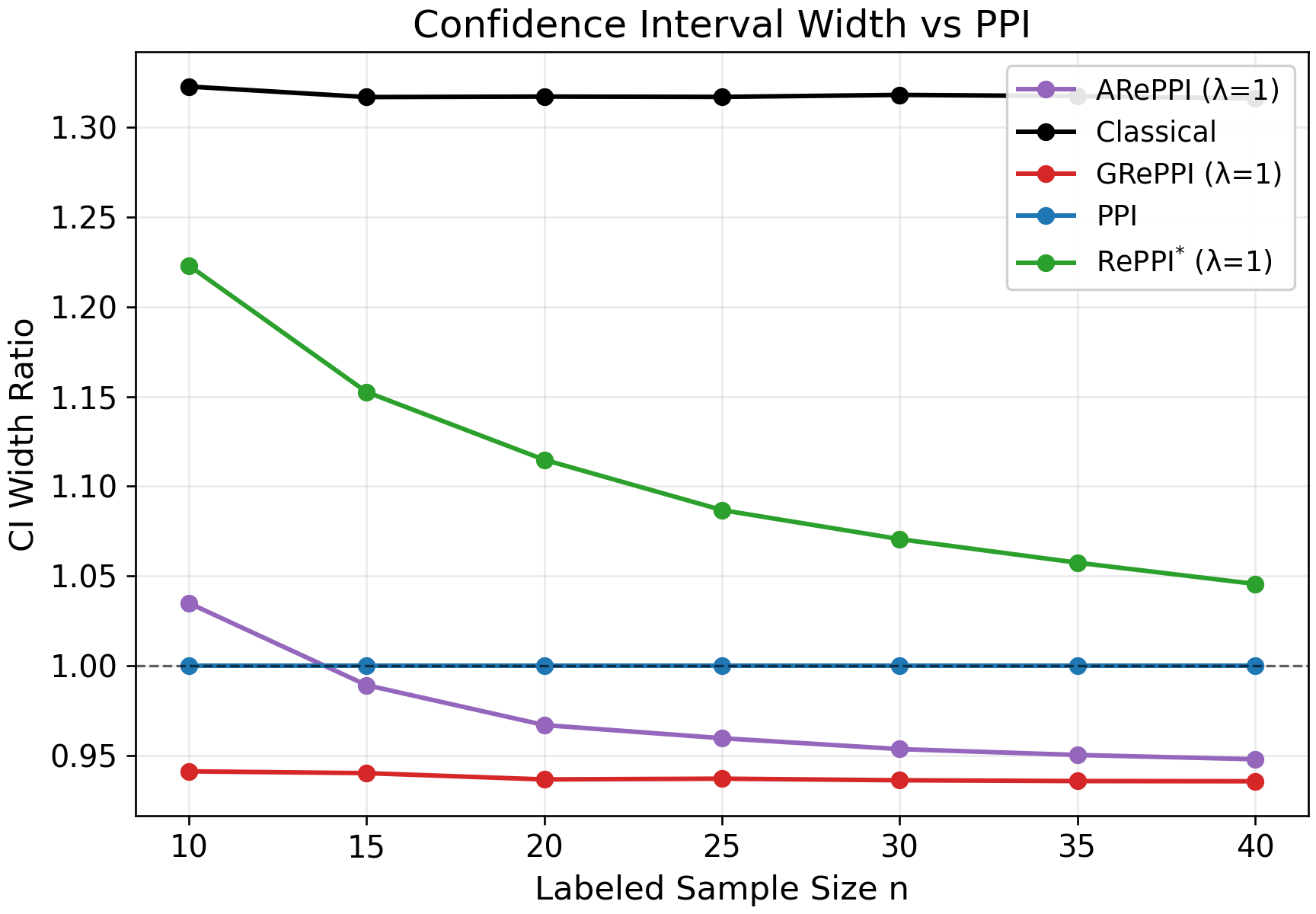}
        \caption{Wald CI Widths without Power Tuning}
        \label{fig:3-sub2}
    \end{subfigure}
    \hspace{1cm}
    \begin{subfigure}[t]{0.4\textwidth}
        \centering
        \includegraphics[width=\linewidth]{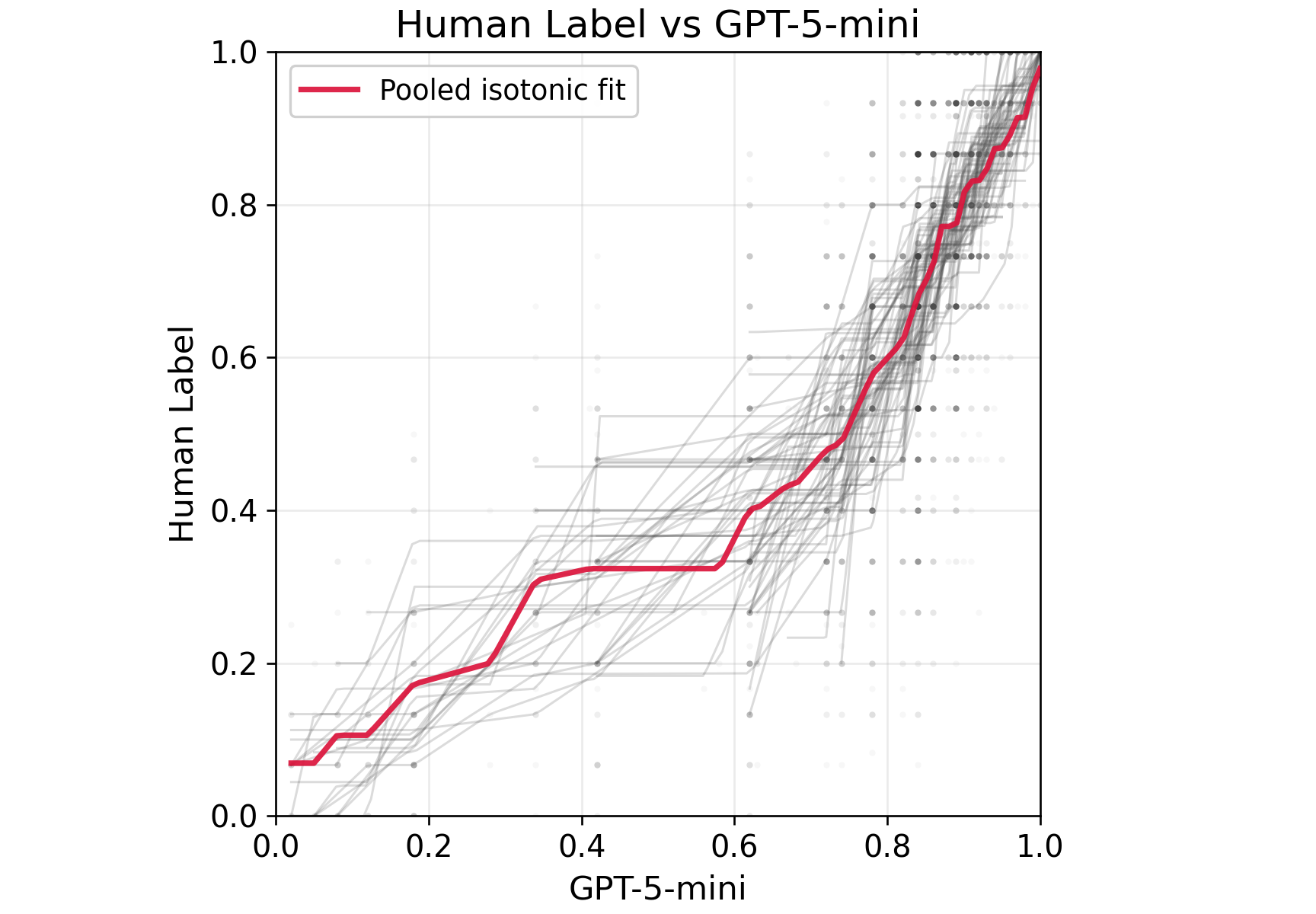}
        \caption{Isotonic Regression of $Y$ onto $\hat{Y}$}
        \label{fig:3-sub3}
    \end{subfigure}    
    \caption{Effects of using our methods (\greppi, \areppi) on a real dataset with human annotations. The plots here are shown without power-tuning (see discussion in Section~\ref{sec:casestudy}). We treat \gptfive predictions as surrogate outcomes, and human annotations as ground truth. Our methods achieve sharper CIs thanks to cross-task sharing of information.} 
    \label{fig:main-3}
\end{figure*}

Prediction-powered inference (PPI)~\citep{PPI-2023} recently popularized a principled way to combine abundant proxy data with limited labeled data while retaining statistically valid confidence intervals, and extended it to more general convex M-estimation problems.  The key idea is to use machine predictions as \emph{surrogate outcomes} and then correct (aka ``rectify'') their bias using the labeled sample. This debiasing step (via the \emph{``rectifier''}) preserves validity even when the proxy is systematically biased. Outside of the machine learning and AI communities, such estimators also have close cousins in the literatures on semiparametric inference and missing data \citep{Robins-1994, Scharfstein1999AdjustingFN, Chen2000AUA}, as well as model-assisted survey sampling \citep{Cassel1976SomeRO, Sarndal2003ModelAS}.\footnote{We defer to Section~\ref{appendix:related-work} for a discussion on the related work and on how our work places among these different threads of literature, as well as the surveys of \cite{kluger} and \cite{Mozer2026PPIIT} for more information on tracing back the historical roots of PPI.} However, existing PPI-style methods leveraging black-box AI predictions are typically applied \emph{separately} within evaluations. Their efficiency therefore strongly depends on having enough labeled observations within each task to estimate the rectifier / debiasing step accurately. In a very sparse-label regime, the variance in the rectification step becomes the primary statistical bottleneck\footnote{In a finite-population setting, it is the \emph{only} statistical bottleneck.}.

At the same time, \emph{tasks in AI evaluation and social science studies are rarely unrelated}: in AI audits, systematic discrepancies between model-based judgments and human annotations can persist across prompts, models, or evaluation criteria. In a similar vein, in political research settings, response patterns often exhibit shared structure across related questions or populations. This shared structure presents an opportunity for researchers and auditors: learn the proxy-ground-truth relationship from other tasks, and then use it to improve inference where labels are scarce. Importantly, this learning step should be done in a computationally lightweight manner: fully retraining the prediction model (increasingly often, a frontier large language model) is infeasible for both AI evaluation tasks and social science research.  

The challenge with lightweight approaches, however, is that naively ``pooling'' information across tasks would invalidate inference. Directly sharing or averaging rectifiers across tasks is biased as soon as discrepancies between the surrogate and the ground-truth are (even mildly) heterogeneous across tasks. More fundamentally, we show that linear cross-task sharing (i.e., only rescaling and shifting proxy scores using information from other tasks) \emph{cannot} improve asymptotic efficiency beyond what is already achievable by PPI++~\citep{PPIpp-2023}, the power-tuned \citep{Rothery1982TheUO, Rubin2008EmpiricalEM} variant of PPI that rescales the proxy estimate to minimize asymptotic variance. Thus, the problem is not simply whether to pool information, but \emph{how} to use the cross-task structure while preserving valid inference for each individual target. We therefore ask:
\begin{center}
{\bf \emph{Can we design PPI-style methods that leverage the shared structure across related tasks to improve statistical efficiency, while retaining valid per-task inference when labeled data are scarce?}} 
\end{center}

\subsection{Contributions and Roadmap} 

\xhdr{Recalibrated PPI Across Many Tasks} We frame AI evaluation as a prediction-powered inference problem across many tasks (Section~\ref{sec:preliminaries})), and propose algorithms that improve inferential power while keeping within-task validity (Section~\ref{sec:methodology}). Our methods build on a simple insight: lightweight recalibration \cite{RePPI-2025} of model outputs can be learned from similar tasks, improving predictive power for the task of interest. Our first method, \greppi, learns a recalibrator  using labeled data from all tasks except the target task, and then performs within-task rectification to obtain valid confidence intervals for the target parameter. This leave-one-task-out construction allows the procedure to exploit shared proxy-ground-truth structure while ensuring that validity remains anchored in target task labels. We introduce a second method, \areppi, which is a cross-validated variant designed to combat possible losses in efficiency when the tasks are, in fact, overly heterogeneous. In this case \areppi will shift to exploit \emph{local} structure over \emph{global} structure. Next, our theoretical analysis clarifies when cross-task surrogate learning can improve traditional PPI with in-task power tuning. In particular, we show that efficiency gains beyond power-tuned PPI require nonlinear structure in the proxy–ground-truth relationship. This constitutes a prescriptive, operational result for practitioners: cross-task learning can yield significant downstream improvements not simply when tasks are similar, but when the shared structure enables recovery of non-negligible nonlinear components of the relationship between surrogates and ground truth. 

\xhdr{Case Study on 2024 US Presidential Election Dataset} For our case study (Section~\ref{sec:casestudy}), we apply our methods to the dataset of~\citep{largescale2025cen}, who collected daily responses from 12 LLMs, including both offline and online-augmented models, to a fixed set of election-related queries over the four months preceding the 2024 U.S. presidential election. \citet[Fig.~4]{largescale2025cen} found (using embedding-based cosine similarity) that model responses to the same election-related queries could differ substantially when the prompt included a demographic or political-affiliation steer. However, embedding similarity leaves open an important question: whether these response differences reflect changes in superficial wording and style, or changes in the substantive content, framing, and factual information conveyed to the user. We refer to this latter notion as \emph{deep-meaning similarity}. This distinction is important because embedding representations can encode many features of text that are not central to the inferential question. Conversely, election-related LLM responses may differ in substantively meaningful ways that embedding similarity alone may not reliably isolate. Indeed, \citet{angwin2024seeking} documented factual errors in LLM responses to election-related questions during the 2024 U.S. presidential election period, that could be hard to detect from embedding similarity. We demonstrate how our method can be employed to infer whether an LLM gives substantively different answers to the same election-related question when it is told different information about the user's demographics or political beliefs, with as little human survey data as possible; Figure~\ref{fig:main-3} presents our results. To obtain said results, we collected a large-scale human annotated sample (details in Section~\ref{sec:casestudy}). Finally, we conduct a series of experiments showing consistent gains of our methods compared to PPI benchmarks on semi-synthetic datasets, and empirically validate our theoretical results. While our paper is designed to give accurate point estimates and confidence intervals, an extension to hypothesis testing---through the duality of testing and CIs---is immediate.

\section{Preliminaries}
\label{sec:preliminaries}

\subsection{Formal Problem Statement}
\label{sec:formalsetup}
We study statistical estimation and inference of parameters $\tasktrue$ for $\task \in \tasks$, with $\abs{\tasks} = \numtasks$. We think of the tasks $\tasks$ as different hypotheses to test, or parameters to estimate and infer, about one or multiple systems, across a (not necessarily, but most likely) common domain. 
Although each task has its own target parameter, tasks often share a common measurement pipeline: ground-truth labels are collected using the same annotation or survey protocol, and proxy outcomes are generated using the same automated scoring procedure. Our methods exploit this shared structure to improve efficiency but, crucially, validity does not rely on the shared-structure assumption. As an example, in our case study (see Section~\ref{sec:casestudy:teaser} for background and Section~\ref{sec:casestudy} for details), the tasks correspond to different steering prompt-pairs of LLMs.

For each task $\task$, we consider a fixed set of $N$ datapoints\footnote{The extension to the case of heterogeneous $N^\taskindex$ is immediate, but we omit it here for notational simplicity.}. Each datapoint is described by covariates $X_i^{\taskindex}$ and has an associated ground-truth outcome $Y_i^\taskindex$. In this paper, we consider the basic setting of learning a collection of mean target parameters, although other extensions are possible for general (convex) $Z$-estimators. Our estimands of interest are therefore described as
\begin{equation}
    \tasktrue = \frac{1}{N}\sum_{i \in [N]} Y_i^\taskindex.
\end{equation}
This finite-population formulation is natural in AI evaluation, benchmarking, and auditing, where the evaluation set is often constructed in advance and treated as the population of interest.
The challenge, however, is that obtaining $Y_i^\taskindex, \forall i \in [N], \task \in \tasks$ can be prohibitively expensive, so we will only observe labels for a subset of points. Specifically, for each $\task \in \tasks$ we will only observe them for the subset $\labeled^\taskindex \subset [N]$, drawn using simple random sampling without replacement from $[N]$. We will also use $O_i^\taskindex = \indicator{i \in \labeled^\taskindex}$. To aid prediction and inference, we assume, however, access to machine-predicted labels $\hat{Y}_i^\taskindex = f^\taskindex(X_i^\taskindex)$ that we will observe for all $i\in[N]$. We note here that \cite{li2025prediction} focus on a similar setup, in a superpopulation setting, and propose an empirical-Bayes inspired improvement over PPI++'s power-tuning mechanism (c.f. Section~\ref{sec:power-tuning}). We will target an orthogonal way to learn from multiple tasks that can be applied independently of theirs, or freely combined\footnote{Albeit losing some of the theoretical guarantees (see experiments in Section~\ref{sec:casestudy}.) We refer to Section~\ref{appendix:related-work} for a more detailed comparison.}.  

\subsection{Motivating Case Study: AI Audits and Social Science Studies}
\label{sec:casestudy:teaser}
To make the abstract setup concrete, we outline here an application that will run through the rest of the paper (deferring details to Section~\ref{sec:casestudy}). There is growing concern \citep{sharma2024towards, perez2023discovering, anwar2024foundational} that LLMs silently personalize their answers when they can infer attributes about the user, distorting the information that users rely on. This concern has only grown since LLMs are increasingly becoming information intermediaries about elections and have been documented to exhibit significant persuasion capabilities for voters~\citep{lin2025persuading, hackenburg2025levers}. Our case study aims to answer the question: \emph{Will an LM answer a \textbf{factual} question about political candidates differently, depending on what it \textbf{assumes} about the user's \textbf{beliefs}?}

We use the data from \citet{largescale2025cen} that includes a set $\calQ$ of predefined questions that were asked to different frontier LLMs from a set $\models$ across a set $\calV$ of prompt variations (see App.~\ref{app:case-study} for details).\footnote{These questions were asked repeatedly over 4 months leading up to the 2024 US Presidential Election. We will only focus on a single date (11/04/2024) and leave the longitudinal considerations for future work. However, for what concerns our methods, one could similarly exploit longitudinal (across time) information as we do here (across tasks). Clearly, annotator bias will be similar across the time dimension too, especially if the annotation protocol is kept consistent over time.} Let $R_{\model,v,q}$ denote the response of model $\model \in\models$ under prompt variation $v\in\calV$ on question $q\in\calQ$. In this setting, a task $\task$ is characterized by a \emph{steering} comparison $(\model, v_1, v_2)$ for a fixed model $\model \in\models$ and prompt variations $v_1 ,v_2 \in \calV$. 
The fixed dataset to evaluate is therefore characterized by the covariates 
$X_i^{\taskindex} = (q_i,\model, v_1, v_2, R_{\model, v_1, q_i},\, R_{\model, v_2, q_i})$ when the steering comparison for task $\task$ is $(\model,v_1,v_2)$. For each covariate $X_i^\taskindex $, we elicit $M$  human annotations, $Y_{i,m}^\taskindex \in [0,1]$ for $m \in [M]$, quantifying how much the response $R_{\model,v_1,q_i}$ differs from $R_{\model,v_2,q_i}$. We assume that $M$ is large enough that the variance in the estimator due to 
$\Var(Y^\taskindex_{i,m} \sep X_i^\taskindex)$ is negligible.
This allows us to treat $Y_i^\taskindex := 1/M\sum_{m \in [M]} Y_{i,m}^{\taskindex}$ as ground truth and ignore annotation variance in the labels. 

\subsection{Background on Surrogate Outcomes}
\label{sec:surrogateoutcomes}
\newcommand{\optimalsurrogate}{g_\theta^*}

Let us focus first on the single-task setting, where $\numtasks=1$, and drop the task superscript. Once again, we assume access to $N$ surrogate outcomes $\hat{Y}_i$ and access to $n$ ground-truth annotations. For this interlude, to illustrate the prior work of \cite{RePPI-2025}, let us assume that $\thetastar$ is the minimizer of a (finite) population target loss $\ell_\theta(X,Y)$. Since $Y$ is only observed on the labeled subset $\labeled$, we cannot directly minimize the full-benchmark average of $\ell_\theta$. To address this in a more efficient way, \citet{RePPI-2025} note that recent PPI estimator proposals can be written as 
\begin{equation}\label{eq:reppi-obj}
    \hat{\theta}_g^{\textsc{ppi}} = \arg\min_{\theta} \underbrace{\frac{1}{n}\sum_{i\in \labeled} \ell_\theta(X_i,Y_i) - \Bigg[\frac{1}{n}\sum_{i\in \labeled} g_\theta(X_i,\hat{Y}_i)}_{\text{"Rectifier"}} - \frac{1}{N}\sum_{i \in [N]} g_\theta(X_i,\hat{Y}_i)\Bigg].
\end{equation}
where $g_\theta (X, \hat{Y})$ is an imputed loss.\footnote{We note that this is a translation of \citep{RePPI-2025} Equation (4), to our finite-population setting. In the superpopulation setting, we typically assume independence between the labeled and unlabeled covariates, so the datapoints used in the first two sums would be non-overlapping with the points used in the last two sums.} The crucial insight needed for our paper is that all $g_\theta$ functions lead to (asymptotically) valid inference\footnote{The asymptotic nature is because the plug-in estimate used for power-tuning (c.f. Section~\ref{sec:power-tuning}) typically uses the same data as the rectifying/debiasing step.}, but they are not all equal; some give rise to better inferential power. \citet{RePPI-2025} adapt a proof from \cite{Robins-1994} to characterize under mild regularity conditions the asymptotically optimal imputed loss. In our finite-population setting, this result can be adapted to show that any loss $\optimalsurrogate$ whose score is proportional to the conditional score of $\ell$ at $\theta = \thetastar$ is asymptotically optimal, i.e.,
\begin{equation}\label{eq:optimalsurrogate}
    \nabla_\theta \optimalsurrogate(X,\hat{Y}) \propto \E[\nabla_\theta \ell_\theta(X,Y) \sep X, \hat{Y}],
\end{equation}
at $\theta = \thetastar$\footnote{Regularity and tools from asymptotic statistics allow the requirement to be relaxed to only hold at $\thetastar$, instead of over all
$\theta$.} is asymptotically optimal. To be precise, this means that $\sqrt{n}(\hat{\theta}^{\textsc{ppi}}_{\optimalsurrogate} - \thetastar) \to \mathcal N(0, \sigma^{\textsc{ppi}}_{\optimalsurrogate})$ with $\sigma^{\textsc{ppi}}_{\optimalsurrogate} \leq \sigma^{\textsc{ppi}}_g$ for any other imputed loss $g$. They propose approximating this score via flexible machine learning methods, an approach they coin as \emph{recalibration}. However, their algorithm requires careful sample splitting, which tends to be impractical within each task, at the small sample sizes we operate in. These observations motivate our idea: rather than learning the recalibration map using only labels from the target task, we ask whether it can be learned from labels in other related tasks while preserving valid inference for the target task.

\subsection{Power Tuning and the Variance Functional $V(s,\lambda)$}
\label{sec:power-tuning}
Power tuning \citep{PPIpp-2023, miao:diagonal} (closely related to optimal control variates \citep{Rothery1982TheUO, Davidson1992RegressionbasedMF, Chen2000AUA} as well as empirical efficiency maximization \citep{Rubin2008EmpiricalEM}) is used to robustify the rectified estimator against misspecification of $g_\theta$ by rescaling the gradient in Eq.~\eqref{eq:optimalsurrogate} by a scalar $\lambda \in \R$, that minimizes a plug-in version of the variance of $\hat{\theta}^{\mathrm{PPI}}_g$. Intuitively, $\lambda$ controls how strongly the estimator \emph{trusts} the surrogate. When $\lambda = 0$, the estimator ignores the surrogate and reduces to the classical labeled-sample mean. When $\lambda$ is large, the estimator relies more heavily on the surrogates. We focus on mean estimation under the finite-population setup of Sec.~\ref{sec:formalsetup}: with the squared loss $\ell_\theta(X,Y) = \tfrac12 (Y-\theta)^2$ we have $\nabla_\theta \ell_\theta(X,Y) = \theta - Y$. The first-order condition of Eq.~\eqref{eq:reppi-obj} along with the characterization of the efficient $g_\theta$ in 

Eq.~\eqref{eq:optimalsurrogate} imply that the optimal surrogate gradient should be instantiated with
$$
    \nabla_\theta g_\theta(X, \hat{Y}) = \lambda(\theta - s(X, \hat{Y})),\quad \text{where } s : (X, \hat{Y}) \mapsto \E[Y \sep X, \hat{Y}].
$$
 
In practice, we will estimate $s$ using flexible machine learning methods, possibly incorporating domain specific knowledge. In our LLM-as-a-judge setting, we will employ isotonic regression, since we empirically observe that ground truth labels tend to follow a monotonic transformation of the LLM scores. Since the $\theta$-terms inside the rectifier cancel, the $\lambda$-tuned estimating equation yields
\begin{equation}\label{eq:fp-estimator}
    \hat{\theta}_\lambda
    =
    \frac{1}{n}\sum_{i \in \labeled} Y_i
    + \lambda\left\{
        \frac{1}{N}\sum_{i \in [N]}  s(X_i,\hat{Y}_i)
        -
        \frac{1}{n}\sum_{i \in \labeled}  s(X_i,\hat{Y}_i)
    \right\}.
\end{equation}
This leaves the question of how to choose $\lambda$. In Appendix~\eqref{appendix:oracle-variance}, we prove that by appealing to results on i.i.d. sampling of $n$ elements from a population of $N$ elements that
\begin{equation*}
    V(s, \lambda)
    :=
    \Var[\hat{\theta}_\lambda]
    =
    \frac{1}{n}\left(1-\frac{n-1}{N-1}\right)
    \Bigl[\sigma_Y^2 - 2\lambda\,\Cov_N(Y, s) + \lambda^2\,\Var_N(s)\Bigr],
\end{equation*}
where $\Var_N, \Cov_N$ denote finite-population variance/covariance. Minimizing over $\lambda$ gives $\lambda^\star(s) = \Cov_N(Y, s)/\Var_N(s)$ and the oracle power-tuned variance
\begin{equation}\label{eq:Vstar}
    V^\star(s)
    =
   \frac{1}{n} \left(1-\frac{n}{N}\right)
    S_Y^2 \bigl(1-\rho_N^2(Y, s)\bigr),
\end{equation}
where $S_Y^2 = \frac{1}{N-1}\sum_{i \in [N]} (Y_i - \thetastar)^2$ is the sample variance. $V^\star(s)$ depends on $s$ only through $\rho_N^2(Y, s)$, the correlation between the true outcome $Y$ and the surrogate outcome $s$. The intuition behind this is easy to see. When $Y$ and $s$ are more strongly correlated, power-tuned PPI~\citep{PPIpp-2023} reduces asymptotic variance more significantly. Correlation exactly (but only) measures the strength of linear relationships, which explains why---as we will see formally later Section~\ref{sec:methodology}---exploiting nonlinear structure when choosing $s$ is going to be crucial. We refer to $V^\star$ as an \emph{oracle} because it needs knowledge of the full labeled dataset to evaluate. In practice, we will use plug-in estimation for this quantity, namely
$
    \hat{\lambda}_{\labeled} =  \frac{{\operatorname{Cov}}_{\labeled} \!\left(Y, s\right)
    }{
        \Var_{\labeled}
        (s)
    }$,
i.e. the (random) plug-in estimate over the labeled subset $\labeled \subset [N]$.

\section{Methodology}\label{sec:methodology}
Prior work on recalibration \cite{RePPI-2025} -- while targeting asymptotic efficiency --- is unsuitable in the \emph{extreme} small data regime we are interested in in our case study. Our per-task labeled budget is extremely small ($n_\task$ as low as $20$ to $40$ samples), so within-task RePPI's data splitting leaves $\approx \!n_\task/2$\footnote{A three-way split like in the original paper is not necessary in the mean estimation case.} samples to fit a nonparametric recalibrator that depends on both the surrogate outcomes as well as the (textual) covariates. However, there is hope. In our repeated evaluation and testing scenario, it is believable that the relationship between human judgement and machine judgement is consistent, regardless of the hypothesis being tested. Throughout, we will refer to $\dataset^{\excepttask} = \cup_{j \in \tasks, j\not=\task } \dataset^{(j)}$ as the leave-one-out-dataset and to $\labeled^{\excepttask} = \cup_{j \in \tasks, j\not=\task } \labeled^{(j)}$ as the LOO-labeled dataset.

\begin{algorithm}[t!]
\caption{\greppi}
\label{alg:loo-reppi}
\begin{algorithmic}[1]
\Require Task Datasets
$\setof{\dataset^\taskindex}_{\task \in \tasks} = \{(\hat{Y}_i^\taskindex,Y_i^\taskindex,O_i^\taskindex)_{i=1}^N\}_{\task \in \tasks}$,
recalibration class \(\mathcal H\).
\For{\(\task=1,\ldots,\numtasks\)}
    \State Fit \(\hat{s}^\excepttask\in\mathcal H\) on \(\labeled^\excepttask\).
    \State Compute power tuning parameter $$\hat{\lambda}_{\labeled^\taskindex} = \frac{{\operatorname{Cov}}_{\labeled^\taskindex} \!\left(Y^\taskindex, \hat{s}^\excepttask(\hat{Y}^\taskindex) \right)
    }{
        \Var_{\labeled^\taskindex}
        (\hat{s}^\excepttask (\hat{Y}^\taskindex))
    } \quad \text{or} \quad  \hat{\lambda}_{\labeled^\taskindex} = 1
    $$
    \State Compute Estimator
    $$
     \hat{\theta}^\taskindex = \frac{1}{n} \sum_{i \in \labeled^\taskindex} Y_i^\taskindex - \hat{\lambda}_{\labeled^\taskindex}\bigg(\frac{1}{n}\sum_{i\in \labeled^\taskindex} \hat{s}^\excepttask \left(\hat{Y}_i^\taskindex\right) -  \frac{1}{N} \sum_{i \in [N]} \hat{s}^\excepttask \left(\hat{Y}_i^\taskindex\right)\bigg)
    $$
\EndFor
\end{algorithmic}
\end{algorithm}
\vspace{-0.2cm}

\xhdr{\greppi}
When the practitioner assumes homogeneity in annotator biases, she might aim to exploit similarities between tasks to improve the predictions on the task of interest. As we will show later, linear transformations of the scores have no effect when within-task power-tuning is applied, so we will typically use a nonlinear transformation. In our experiments, we will instantiate this with isotonic regression fit purely on $\hat{Y},Y$ pairs -- that is, ignoring covariate information -- for its simplicitly, and because LLM judgements are known to often be miscalibrated \citep{xion:overconfident, detommaso:multicalibration}.\footnote{Experimentally, we observed similar results with more general nonparametric approaches such as XGBoost \cite{guestrin}.}. This gives rise to a particularly simple  first method that fits $s$ on $\labeled^{\excepttask}$ and then tunes $\lambda$ on $\labeled^\taskindex$ to build an improved PPI estimator on $\dataset^\taskindex$. The resulting algorithm \emph{\textsc{G}lobally \textsc{Re}calibrated \textsc{P}rediction-\textsc{P}owered \textsc{I}nference} (\greppi) is given in Algorithm~\ref{alg:loo-reppi}. We note that similarly to \ppipp this estimator is not finite-sample unbiased and its residuals underestimate the true variance (yielding miscoverage); to combat this, power tuning can be skipped, or $\lambda$ can be tuned on a different subset of the labeled data than the rectifier is evaluated (c.f. Section~\ref{sec:casestudy}).

\xhdr{\areppi}
When tasks are in fact homogeneous (in the sense that annotation bias is common), Algorithm~\ref{alg:loo-reppi} leads to substantial improvements in in-task MSE, in-task coverage and in-task Asymptotic CI widths, as we illustrate in Section~\ref{sec:casestudy}. However, when tasks are heterogeneous, there is no safeguard (beyond the mechanism \ppipp offers) against poor predictions of the recalibration function $\hat{s}$, potentially degrading beyond the  \ppipp (\ppipp can be thought of restricting $\mathcal{H}$ to the singleton containing the identity and applying power tuning as above). To that end, instead of fitting $\hat{s}$ purely on all other tasks and hoping for homogeneity, we could include data from the task itself in the recalibration procedure. This safeguards against heterogeneity at the cost of some finite-sample efficiency setbacks at small $n$.
The algorithm is given in Algorithm~\ref{alg:adaptive-reppi}. The idea is as follows. First, we fit a global $\hat{s}^\excepttask$ (Line 2). Then, we fit the local data into two folds (Line 3). One fold is used to learn an adaptive $\hat{s}^{\mathrm{ada}}$ (Line 7) used to make prediction on the other fold (Line 9) by learning first a local $\hat{s}^{\mathrm{oof}}$ (Line 5) (using cross-validation, or LOO predictions, in order to avoid fitting $\hat{s}$ on the labels for which we get surrogate predictions). Then the global and local $\hat{s}$ are linearly combined (Line 7) to  maximize the correlation (Line 6) of the combined recalibrator with the outcomes (in the attempt to minimize downstream oracle variance after power tuning). This gets repeated with the other data split, and predictions are made for the complementary split (Line 9)\footnote{See also Cross Prediction-Powered Inference \cite{CrossPPI-2023}}. These predictions are finally used for power tuning (Line 10), whenever power tuning is used (otherwise $\lambda=1$). On the unlabeled data, we use the predictions given by averaging both adaptive recalibrators (Line 11). The final estimator is computed in Line 12.
\begin{algorithm}[t!]
\caption{\areppi}
\label{alg:adaptive-reppi}
\begin{algorithmic}[1]
\Require $\setof{\dataset^\taskindex}_{\task \in \tasks} = \{(\hat{Y}_i^\taskindex,Y_i^\taskindex,O_i^\taskindex)_{i=1}^N\}_{\task \in \tasks}$,
recalibration class \(\mathcal H\), number inner folds $K$.
\For{\(\task=1,\ldots,\numtasks\)}
\State Fit $\hat{s}^{\excepttask}\in\mathcal H$ on $\labeled^{\excepttask}$.
\State Split $\labeled^\taskindex$ into two folds $A$ and $B$.
\For{$F \in\{A,B\}$}
    \State Obtain OOF predictions $\hat{s}_F^{\mathrm{oof}}$ on $F$ by $K$-fold CV fit in $\mathcal H$.
    \State Choose
    \[
    \hat{\gamma}_F
    \in
    \argmax_{\gamma\in[0,1]}
    \rho_F^2\!\left(
    \gamma \hat{s}_F^{\mathrm{oof}}(\hat{Y}_i^\taskindex)
    +(1-\gamma)\hat{s}^{\excepttask}(\hat{Y}_i^\taskindex),
    Y_i^\taskindex
    \right)
    \]
    \State Refit $\hat{s}_F\in\mathcal H$ on all of $F$ and define
    \[
    \hat{s}^{\mathrm{ada}}_F
    =
    \hat{\gamma}_F \hat{s}_F
    +(1-\hat{\gamma}_F)\hat{s}^{\excepttask}.
    \]
\EndFor
\State OOF prediction on $\labeled^\taskindex$ as $u_i^\taskindex = s^{\mathrm{ada}}_B(\hat{Y}_i^\taskindex)$ for $i\in A$ and
$u_i^\taskindex = s^{\mathrm{ada}}_A(\hat{Y}_i^\taskindex)$ for $i\in B$.
\State Compute power tuning Factor $\hat{\lambda}_{\labeled^\taskindex}^{\mathrm{oof}}$ with $u_i^\taskindex$ in place of $\hat s_i^\excepttask(\hat{Y}^\taskindex)$
$$
\hat{\lambda}_{\labeled^\taskindex}^{\mathrm{oof}} = \frac{{\operatorname{Cov}}_{\labeled^\taskindex} \!\left(Y^\taskindex,  u_i^\taskindex \right)
    }{
        \Var_{\labeled^\taskindex}
        (u_i^\taskindex) 
    } \quad \text{or} \quad  \hat{\lambda}_{\labeled^\taskindex} = 1.
$$
\State Let $u_i^\taskindex = \frac{1}{2}\left(s^{\mathrm{ada}}_A(\hat{Y}_i^\taskindex) + s^{\mathrm{ada}}_B(\hat{Y}_i^\taskindex)\right)$ for $i \in [N] \setminus \labeled^\taskindex$.
\State Compute estimator
    \[
    \hat{\theta}^\taskindex
    =
    \frac{1}{n}\sum_{i\in \labeled^\taskindex}
    {Y}_i^\taskindex
    -\hat{\lambda}_{\labeled^\taskindex}^{\mathrm{oof}}
    \left(
    \frac{1}{n}\sum_{i\in \labeled^\taskindex}
    u_i^\taskindex
    -
    \frac{1}{N}\sum_{i \in [N]}
    u_i^\taskindex\right).
    \]
\EndFor
\end{algorithmic}
\end{algorithm}

\subsection{Geometric Insights: Nonlinearity of Recalibration}
\label{sec:nonlinearity}

\begin{figure*}[t]
\label{figure:main:synthetic}
    \centering
    \begin{subfigure}[t]{0.24\textwidth}
        \centering
\includegraphics[width=\linewidth]{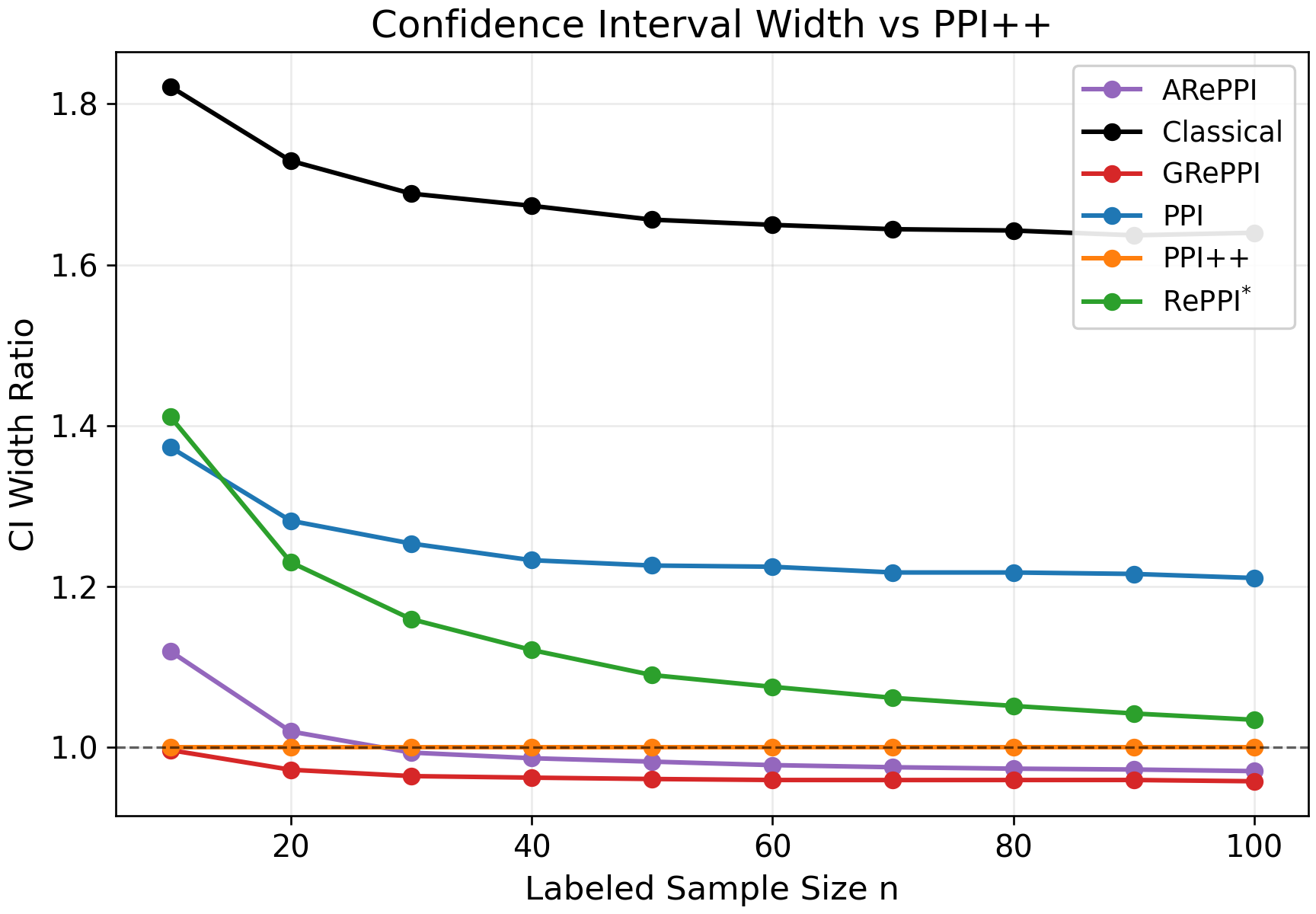}
        \caption{CI Widths}
        \label{fig:1-sub1}
    \end{subfigure}
    \hfill
    \begin{subfigure}[t]{0.24\textwidth}
\centering\includegraphics[width=\linewidth]{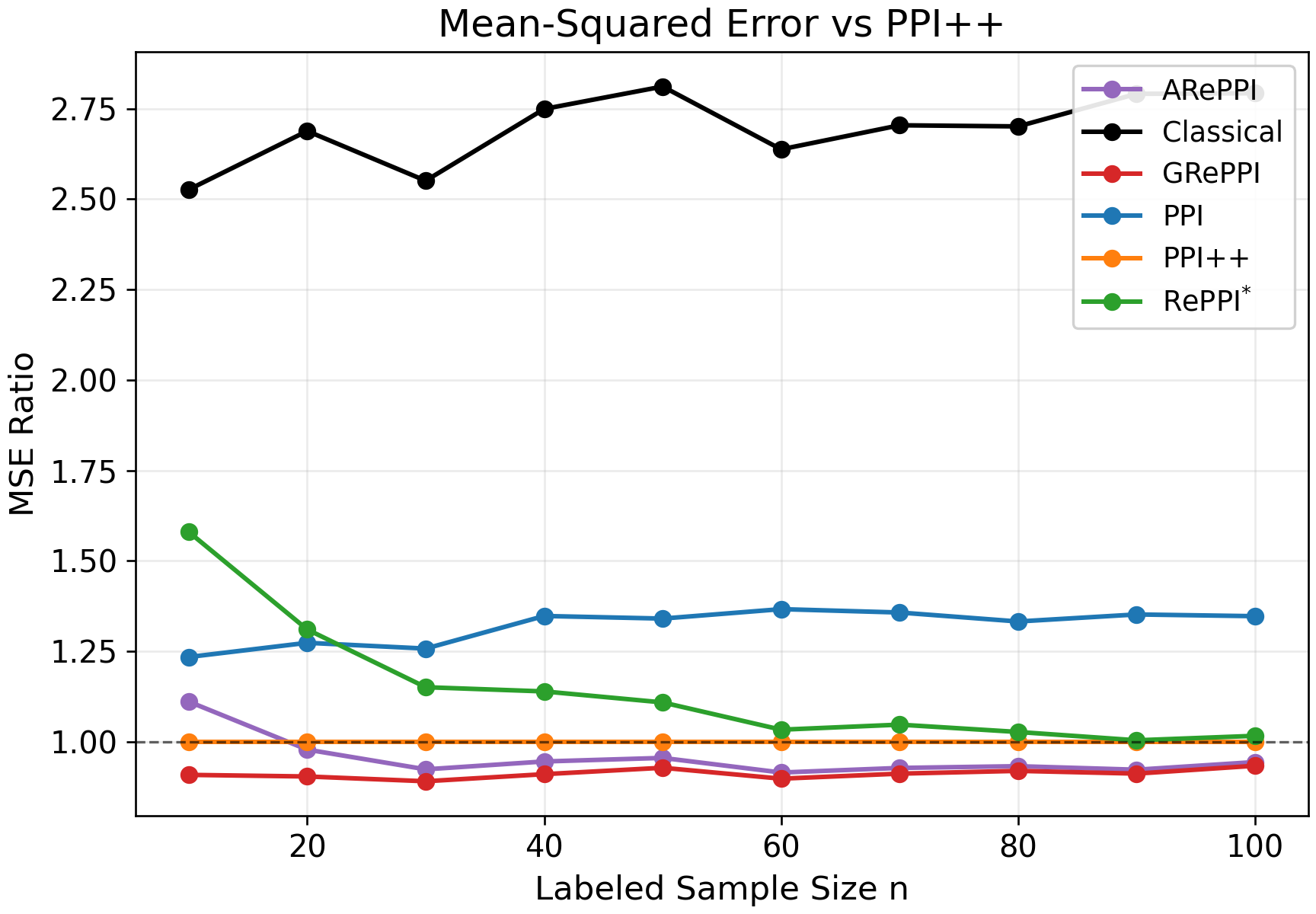}
        \caption{MSE}
        \label{fig:1-sub2}
    \end{subfigure}
\hfill
   \begin{subfigure}[t]{0.24\linewidth}
        \centering
        \includegraphics[width=\linewidth]{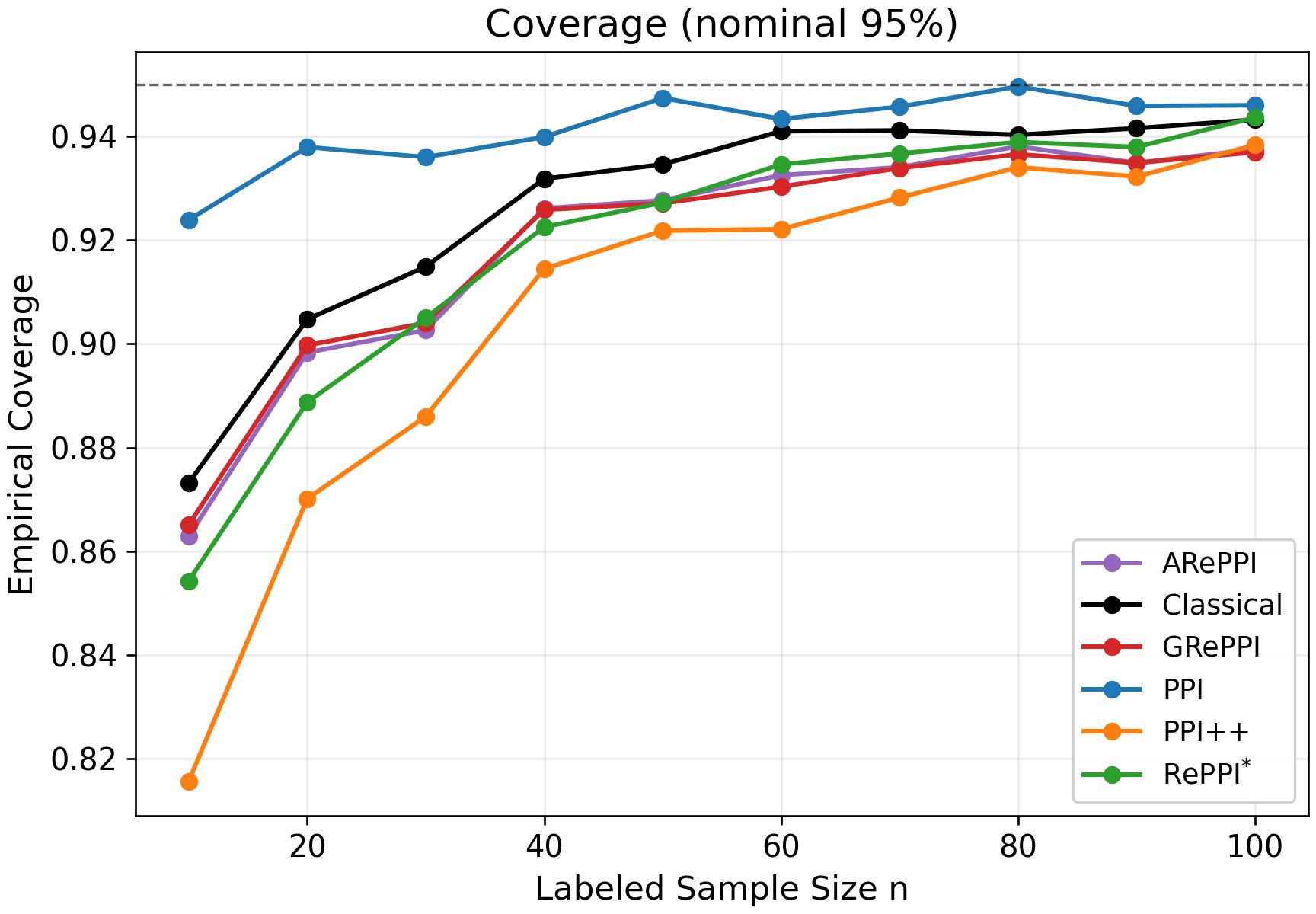}
        \caption{Coverage}
        \label{fig:1-sub3}
    \end{subfigure}
    \hfill
    \begin{subfigure}[t]{0.24\linewidth}
        \centering
        \includegraphics[width=\linewidth]{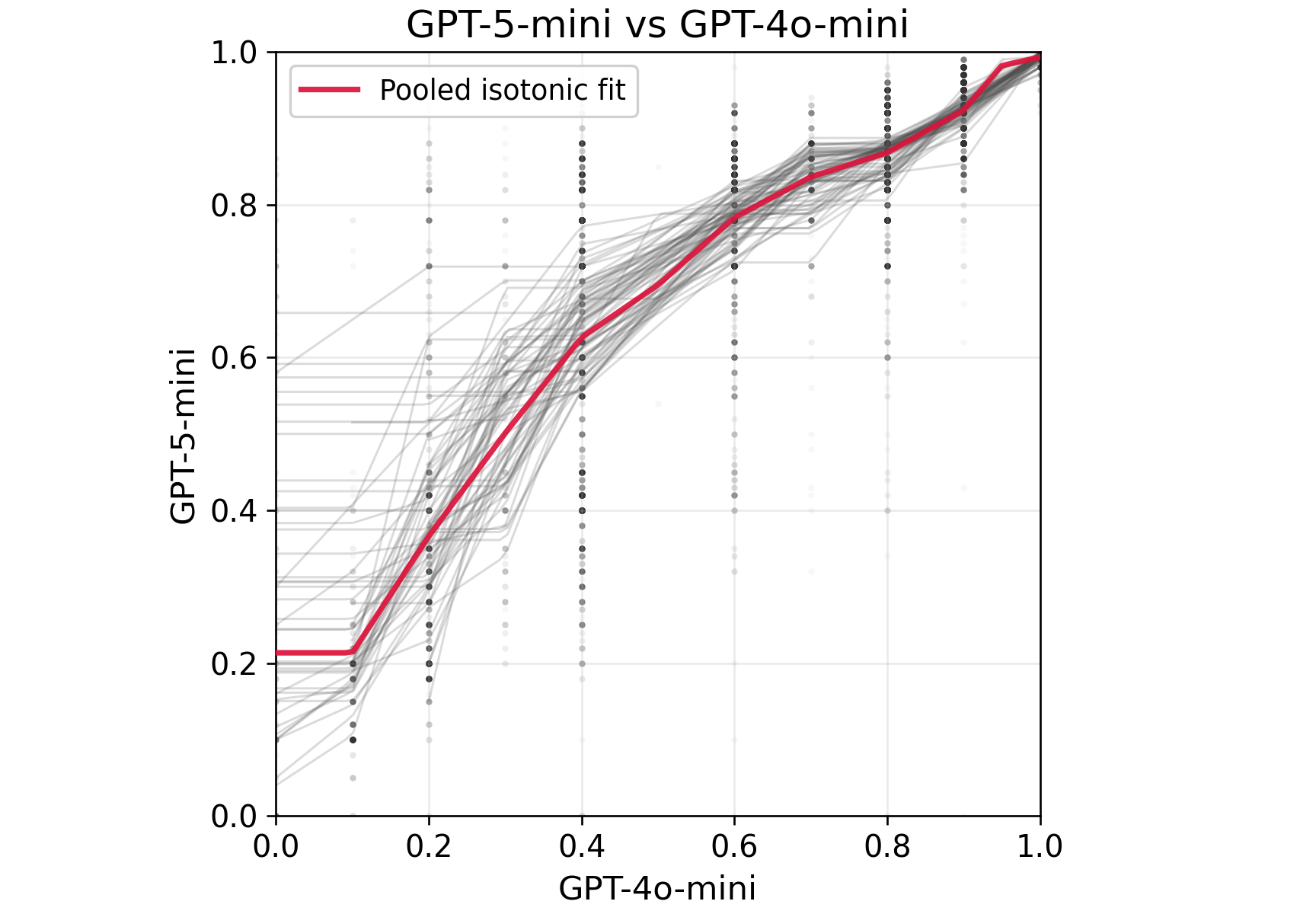}
        \caption{Relationship}
        \label{fig:1-sub4}
    \end{subfigure} 
    \hfill
    \caption{Semi-synthetic data experiment, with \gptfive as ground truth and \fouromini as the annotator. For plots (a)-(b), we treat \ppipp as the reference, and plot all other methods in a relative sense. We evaluate \classical, \ppipp, \reppi, \greppi and \areppi with local power tuning.}  
    \label{fig:main-synthetic}
\end{figure*}

We comment here on some requirements on $\mathcal{H}$. Motivated by the original insight of PPI, a first idea on how to leverage multi-task data could be to first compute a rectifier using all the labeled data $\labeled$, and then ``Rectifying the Rectifier'' using in task data $\labeled^\task$. Carrying out this calculation shows that the pooled rectifier cancels out again, leaving no benefits over standard \ppi. As we see below, using power tuning makes the requirements on the recalibration class $\mathcal{H}$ even more stringent. In fact, any affine transformation of the in-task surrogates $\hat{Y}^\taskindex$ can be shown to have no potential for oracle variance improvement. 

\begin{lemma}[Affine invariance]\label{lem:affine}
Let $s \in \mathcal{H}_{\mathrm{affine}}$, that is  $\hat{s}(\hat{Y}) = a\hat{Y} + b$ and let $\mathrm{id}: \hat{Y} \mapsto \hat{Y}$. Then, $V^\star(\hat{s}) = V^\star(\mathrm{id})$.
\end{lemma}

\vspace{-10pt}
\begin{proof}
Pearson correlation is invariant under nonzero affine transformation of either argument: $\Cov_N(Y, a \hat{Y} + b) = a\Cov_N(Y,\hat{Y})$ and $\Var_N(a \hat{Y}+b) = a^2\Var_N(\hat{Y})$, so $\rho_N^2(Y, a\hat{Y}+b) = \rho_N^2(Y, \hat{Y})$. Plugging into~\eqref{eq:Vstar} gives $V^\star(\hat{s}) = V^\star(\mathrm{id})$.
\end{proof}
The following corollary (with proof in \Cref{appendix:nonlinear-necessity}) strengthens this observation: any improvement over the identity surrogate must come from \emph{nonlinear} structure in the proxy--outcome relationship, whenever power-tuning acts only locally.

\begin{corollary}[Nonlinear necessity]\label{cor:nonlinear}
Let $m(z) \defeq \E_N[Y \sep \hat{Y} = z]$ denote the (finite-population) regression of $Y$ on the proxy. For any measurable $\phi:\R\to\R$, strict improvement $V^\star(\phi) < V^\star(\mathrm{id})$ is possible if and only if $m$ is not affine on the population support $\{\hat{Y}_i:i\in[N]\}$. The infimum is attained by $\phi^\star(z) = m(z)$, and the maximum achievable gain is
\begin{equation}\label{eq:max-gain}
    V^\star(\mathrm{id}) - \inf_{\phi} V^\star(\phi) \;=\; \frac{1}{n}\left(1-\frac{n}{N}\right)S_Y^2 \bigl(R^2_{Y\sim \hat{Y}} - \rho_N^2(Y, \hat{Y})\bigr),
\end{equation}
with $R^2_{Y\sim \hat{Y}} := \Var_N(\E_N[Y\sep \hat{Y}])/\Var_N(Y)$ being the finite-population (nonparametric) regression $R^2$.
\end{corollary}
We note that analogous results can readily be derived for the superpopulation setting.

\section{Case Study}
\label{sec:casestudy}

\subsection{{Experimental Setup: Demographic Steering of LLMs}}

\xhdr{Dataset} %
Our case study focuses on the dataset of~\citet{largescale2025cen} who queried 12 LLMs using a set of 572 base questions and 22 prompt variations daily over a period of roughly 4 months. Among all the queried LLMs in their dataset, we restrict our attention to the offline versions of three prominent models: GPT-4o, Claude-3.5-Sonnet, and Gemini-1.0-Pro. These constitute our set $\calM$. Among the total ``base'' questions in the dataset, we focus on 186 that directly pertain to the election processes and the candidates, and we exclude exit poll and prediction questions as they mostly elicit refusals from the LLMs; these 186 questions constitute our $\calQ$.

\begin{figure*}[t!]
\label{fig:main:synthetic}
    \centering
    \begin{subfigure}[t]{0.24\textwidth}
\includegraphics[width=\linewidth]{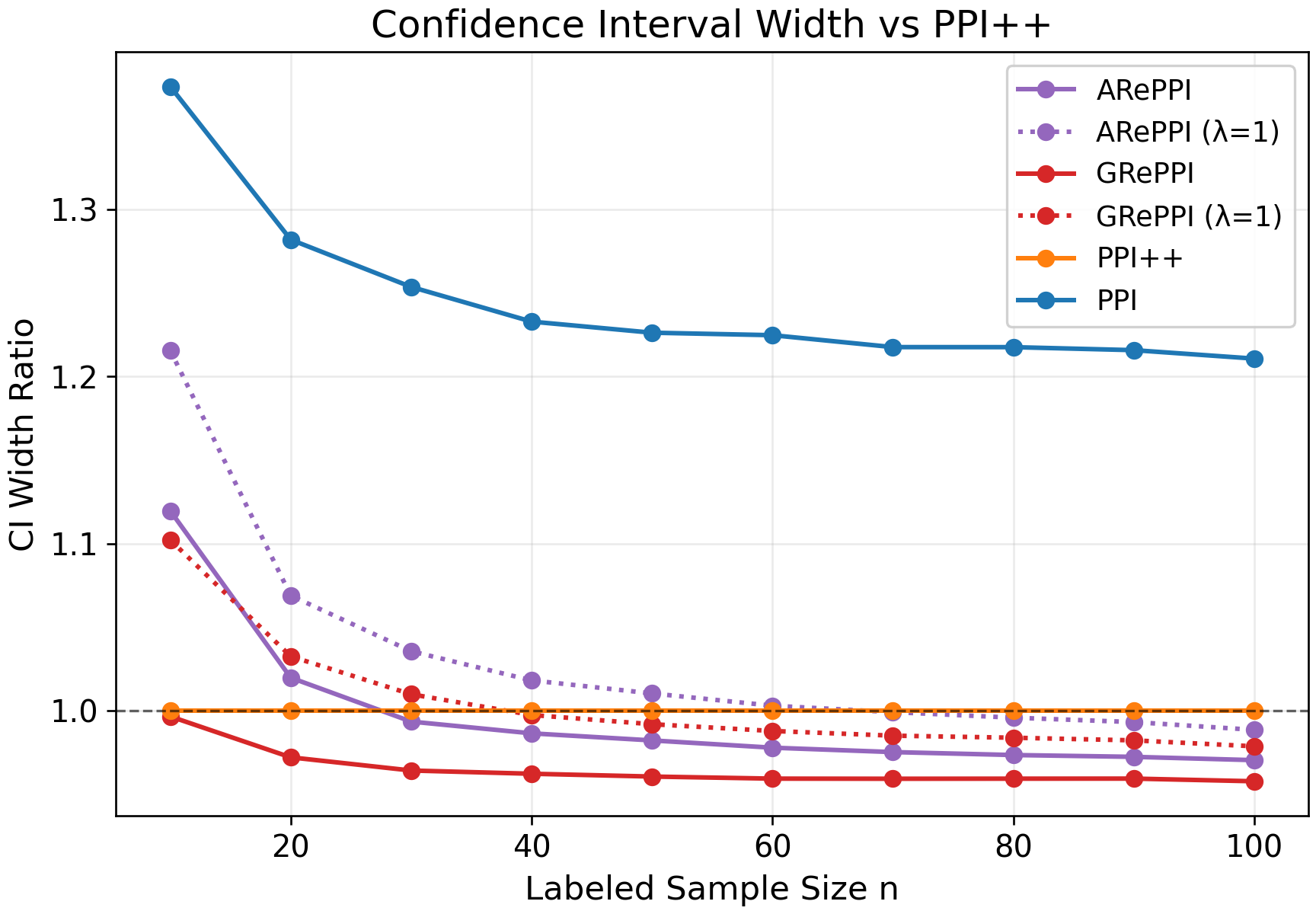}
        \caption{CI Width}
        \label{fig:2-sub1}
    \end{subfigure}
    \hfill
    \begin{subfigure}[t]{0.24\textwidth}
\centering\includegraphics[width=\linewidth]{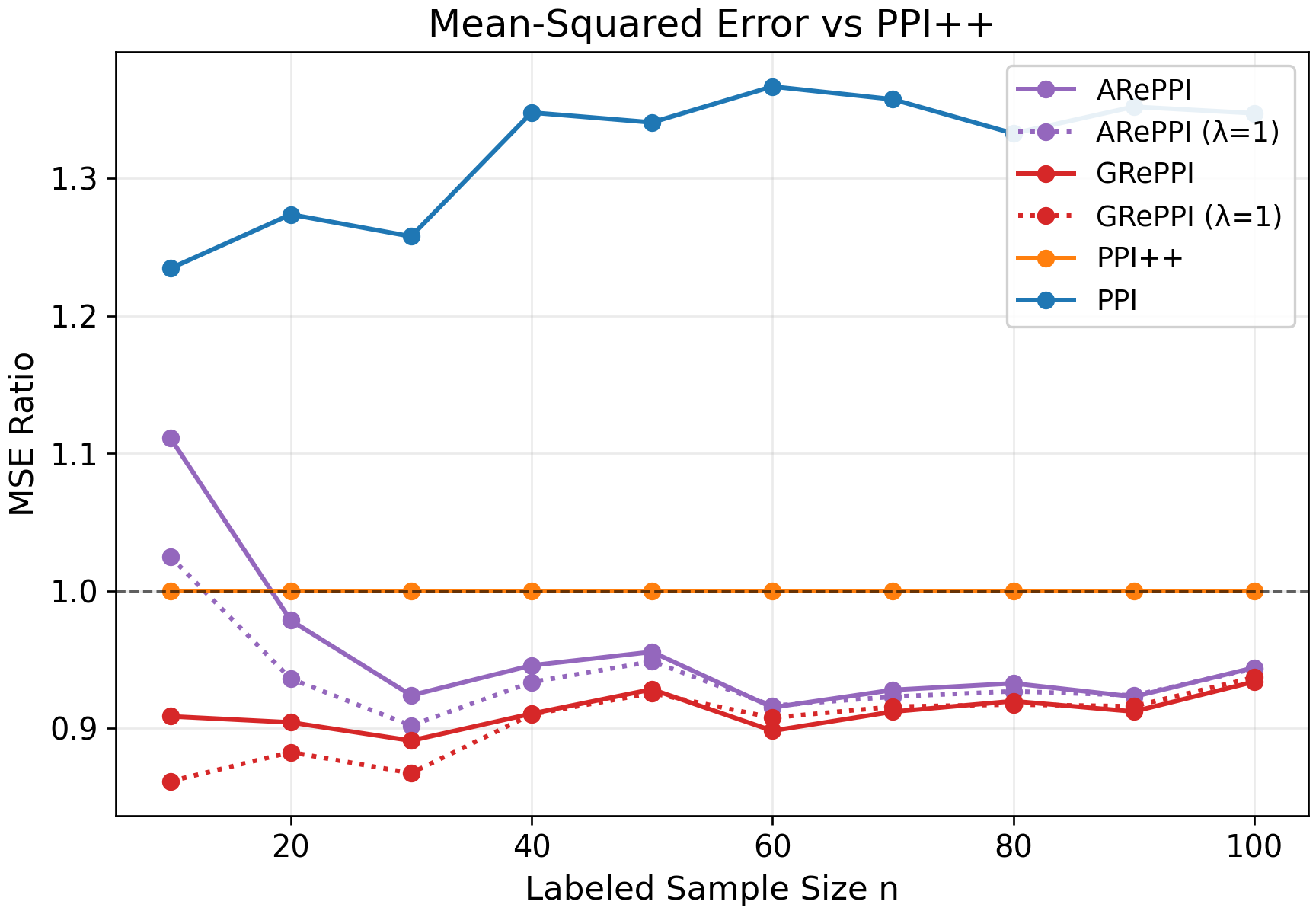}
        \caption{MSE}
        \label{fig:2-sub2}
    \end{subfigure}
\hfill
\begin{subfigure}[t]{0.24\textwidth}
\includegraphics[width=\linewidth]{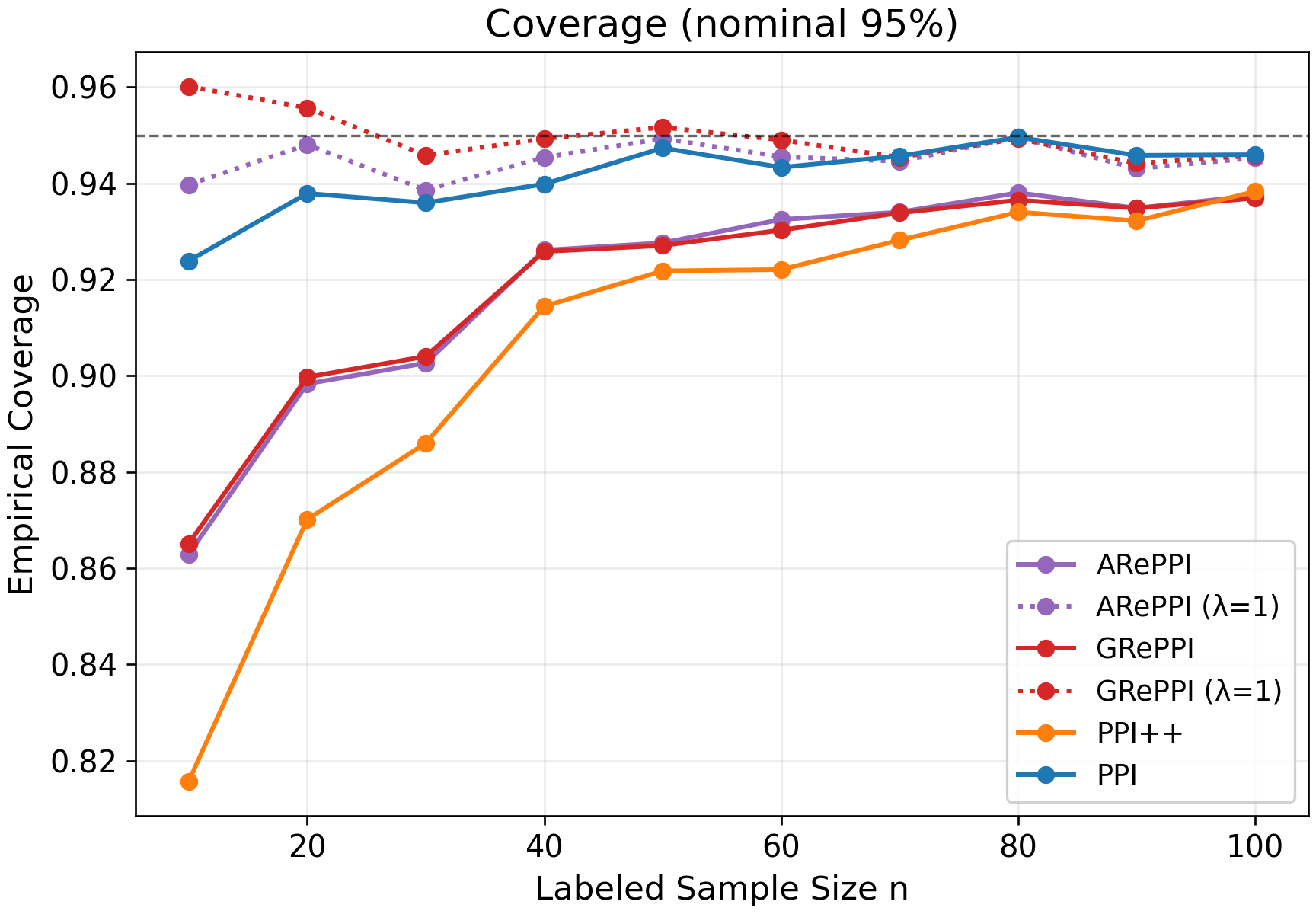}
        \caption{Coverage}
        \label{fig:2-sub3}
    \end{subfigure}
    \hfill
    \begin{subfigure}[t]{0.24\textwidth}
\centering\includegraphics[width=\linewidth]{new_figures/relationship/semi_synthetic.png}
        \caption{Relationship}
        \label{fig:2-sub4}
    \end{subfigure}
\hfill
    \caption{Same setup as Figure 2. We illustrate the effects of power tuning's underestimation of the variance at small sample sizes and effects on coverage. Figure (c) shows that power tuning has an adverse effect on coverage at the very small sample sizes we operate in, because the residuals underestimate the out-of-distribution residuals. However, Figure (b) shows that the effects on MSE is negligible when out-of-task recalibration is already doing a good job fine-tuning the surrogate predictions, suggesting that power-tuning is not necessary.}  
    \label{fig:main}
\end{figure*}

In the dataset of~\citet{largescale2025cen}, a base question $q \in \calQ$ can be modified by a prompt variation $v$ (e.g., prepending ``I am a Democrat.'') to form a complete query. From the original 22 variations, we isolate 2 distinct subsets to study demographic steering: $\mathcal{P}_\text{race}$, consisting of 6 racial identities plus the ``none'' (unmodified) variation, and $\mathcal{P}_\text{party}$ consisting of 2 political identities plus the ``none'' variation.

We define the full set of tasks (i.e., comparisons between two different prompt variations on a single model) by taking all pairwise comparisons within $\mathcal{P}_\text{race}$ and within $\mathcal{P}_\text{party}$, across all models in $\models$:
\begin{align*}
\tasks = \left\{(\model, v_1, v_2) : \model\in \models \land \left((v_1, v_2 \in \mathcal{P}_\text{race}) \lor (v_1, v_2 \in \mathcal{P}_\text{party})\right) \right\}.
\end{align*}
$|\tasks| = 72$, since there are $\binom{7}{2}=21$ pairs in $\mathcal{P}_\text{race}$,  $\binom{3}{2}=3$ pairs in $\mathcal{P}_\text{party}$, and 3 models to evaluate.

Finally, for a given task $\task$, our dataset consists of 186 data points corresponding to the base questions in $\calQ$. The $i$-th data point $X_i^\taskindex$ is formally denoted as $X_i^\taskindex = \left(q_i, v_1, v_2, \model, R_{\model, v_1, q_i}, R_{\model, v_2, q_i}\right)$, where $q_i \in \calQ$, and $R_{\model, v, q}$ is the fixed textual response generated by the model to the prompt characterized by the base question $q$ and the prompt variation $v$ within the task indexed by $\task \in \tasks$.

\xhdr{``Deep Meaning Similarity''} For each datapoint $X_i^\taskindex$, the ground-truth label $Y_i^\taskindex$ is the \emph{deep meaning similarity} between the two responses $R_{\model, v_1, q_i}$ and $R_{\model, v_2, q_i}$. We define this as a score in $[0,1]$, where 1 indicates that the responses are maximally similar in meaning and 0 indicates that they are maximally dissimilar.\footnote{What constitutes ``deep meaning similarity'' is fundamentally normative and remains debated in political communication, especially in work on framing and interpretation~\citep{entman1993framing,chong2007framing,walter2024meta}. Our goal is not to resolve this debate, but to operationalize a notion of similarity that captures not only factual content, but also tone, emphasis, and framing, while excluding superficial differences in wording, formatting, or style.} For each labeled datapoint, we elicit $M$ independent human annotations $Y_{i,m}^\taskindex$ and define the ground-truth label as their average, $Y_i^\taskindex = \frac{1}{M}\sum_{m=1}^M Y_{i,m}^\taskindex.$ In our study, we set $M=5$, obtaining five annotations per labeled response pair. To evaluate performance across different labeled-data regimes, we collect $n_\task=40$ labeled datapoints for each task $\task$. Since our case study contains 72 tasks, this yields $72 \times 40 \times 5 = 14{,}400$ total human annotations. We collected these annotations through a Prolific survey with 5{,}760 participants, each of whom annotated either two or three response pairs for about \$9/hr. The annotation instrument follows standard survey-design principles~\citep{groves2005-surveybook,tourangeau2000psychology}; full details appear in Appendix~\ref{app:case-study}. In addition to the semi-synthetic and human annotated datasets we evaluate our method on here, we include in Appendix~\ref{app:ablation} a synthetic data ablation showcasing how task heterogeneity affects the performance of these methods. All our experiments are run on an Apple M2 Pro 16GB.

\subsection{{Semi-Synthetic Results}}
In this section, we begin gaining intuition for our problem setting and evaluate the methods from Sec.~\ref{sec:methodology}. As we wish to be able to evaluate against a ground truth, and we do not have a fully labeled dataset of human ground truth, we will first conduct a semi-synthetic version of the case study using two pairs of \emph{annotator models}. We use \gptfive as ground truth $Y$, and use \fouromini to supply the surrogate prediction $\hat{Y}$. The exact prompts for obtaining the scores can be found in App.~\ref{app:semisyntheticscores}. We compare the following methods: (1) $\classical $ (i.e., standard mean estimation using only $\labeled^\taskindex$; (2) \ppi \citep{PPI-2023} (rectify on $\labeled^\taskindex$ with no power tuning) (3) $\ppipp $ \citep{PPIpp-2023} (power-tune and compute rectifier on $\labeled^\taskindex$ and use machine predictions on $\dataset^\taskindex$, i.e. \greppi using $\hat{s} = \mathrm{id}.$); (4) \reppi \citep{RePPI-2025} only on $\labeled^\taskindex$ using a two-way split, and using isotonic regression as the recalibration function class (see Appendix~\ref{appendix:cis:reppi} for a detailed account of our implementation, since it differs slightly from the original version meant for the superpopulation case\footnote{Which is why we denote it as \textsc{RePPI}$^*$ in our plots.}) (5) $\greppi$ (Algorithm~\ref{alg:loo-reppi}); and (6) $\areppi $ (Algorithm~\ref{alg:adaptive-reppi}). Both (5) and (6) use isotonic regression for $\mathcal{H}$. We report in  Fig~\ref{fig:main-synthetic} the (Studentized) Wald Confidence Interval Widths and Coverage and Average Mean Squared Errors (MSE) over $B=100$ draws of the labeled set for each task. Implementation details can be found in Appendix~\ref{app:implementation-details}. We generally observe that both CI width and MSE scales as we would expect (we can see in Fig.~\ref{fig:1-sub4} that the scores in fact have a nonlinear, monotone relationship): \greppi yields the lowest errors and the tightest confidence intervals, because it can exploit the conjectured structure. \areppi also exploits structure, but needs enough samples for the process to stabilize. Both methods beat \textsc{RePPI}, who doesn't leverage cross-task information and suffers in our very sparse data regime. \ppipp cannot exploit nonlinear dependencies learnable from the surrogate to ground-truth mapping. However, we also see that at small sample sizes, coverage suffers across all methods. In Fig.~\ref{fig:main}, we confirm that this is due to local power tuning, and show that fixing $\lambda=1$ improves coverage across the board. We also see that when the recalibration map learned from auxiliary tasks is good, power tuning is not necessary to get low MSE (in fact, the difference is very small). Since power-tuning makes variance estimation harder (due to overfitting), we do not use it for the human data study below.

\subsection{Results with Human Annotators}

We next turn our attention to human annotated data. On our real dataset, we cannot evaluate the ground truth, and therefore MSE and coverage are unobserved metrics. We will therefore only report CI widths in Figure~\ref{fig:main-3}. However, as seen above, since our sample sizes are so small ($n \leq 40$ for the human annotated data, due to the sheer cost of gathering ground truth annotations), power tuning makes variance estimation overly optimistic, as evidenced by the coverage plot Fig.~\ref{fig:2-sub1},~\ref{fig:2-sub3} on the semi-synthetic data. Therefore, CI width using power-tuning is a poor metric, and we choose to restrict to $\lambda=1$ in Figure~\ref{fig:main-3}. We see again the same hierarchy of methods as in the semi-synthetic experiment, with \greppi being particularly helpful in reducing CI widths, and \areppi catching up as more samples are labeled.

\vspace{-5pt}
\section{Discussion}\label{sec:discussion}

In this paper, we introduced a framework to improve power in PPI across many related tasks while keeping validity anchored in the inferential task of interest, motivated by settings in AI evaluation and social science research. Our methods leverage shared nonlinear structure in the proxy-ground-truth relationship across tasks to learn improved recalibrated surrogate outcomes, but keep power-tuning and debiasing task-specific. 

\xhdr{From CIs to Auditor Guidelines} Our methods provide (asymptotically) unbiased estimates and valid CIs for task-level quantities such as the average deep-meaning similarity between responses. Downstream auditing decisions and hypothesis tests typically require translating these continuous measurements into operational judgements. In our case study, an auditor may wish to determine whether a set of pairs of responses is ``meaningfully different'' in content, framing, or factual substance, which in turn requires specifying a similarity threshold below which responses are considered substantively different. Defining such thresholds is fundamentally a normative and application-dependent question: different stakeholders may reasonably disagree about what degree of semantic divergence warrants concern. Moreover, desirable properties of such thresholds, such as robustness, interpretability, calibration to human perception, or consistency across domains, remain largely unexplored. We therefore view the development of principled auditor guidelines and decision thresholds as an important direction for future interdisciplinary work spanning statistics, social science, HCI, and AI governance.

\section{Related Work}
\label{appendix:related-work}
\paragraph{Prediction Powered Inference}
Prediction-powered inference \citep{PPI-2023} (and the related DSL framework~\citep{egami2023using}) has recently been proposed to improve statistical power while retaining valid inference downstream of AI predictions. PPI++ \citep{PPIpp-2023} adds a power-tuning coefficient $\lambda$ that guarantees variance no worse than the classical labeled-only estimator asymptotically (see also \cite{miao:diagonal} for multivariate extensions of power-tuning). PPI has seen domain specific instantiations and extensions, as well as domain-agnostic extensions: We mention here briefly two papers that incorporate different sorts of multiplicity. \cite{StratPPI-2024} partition the data distribution into strata, and select separate power-tuning parameters between for stratum. \cite{MultiPPI-2026} study a setting with multiple available machine annotations in a budget constrained setting. Both of these works are orthogonal to our work: in the former, multiplicity arises within a single population, in the latter it arises through multiple annotators. For us, it arises through multiple tasks/populations/parameters.

\paragraph{Prediction-Powered Adaptive Shrinkage Estimation} \cite{li2025prediction} also consider a multi-task PPI setting, by using an Empirical Bayes procedure to shrink the power-tuning parameters. The key distinction is that shrinkage typically acts on the task-level estimates themselves, improving aggregate MSE but potentially compromising valid inference for any single task (robust confidence intervals are more difficult to obtain when Empirical Bayes shrinkage is involved \cite{armstrong2022robust}). Our methods target a different layer: we borrow information to learn a cross-task proxy recalibration, then keep task-specific rectification, preserving per-task confidence intervals. One can view these two approaches as orthogonal, and they can be readily combined. We view this as future work.

\paragraph{Surrogate Outcomes, Recalibration and Calibration}
 \cite{RePPI-2025} observe that prediction-powered inference with a recalibrated surrogate approaches a semiparametric efficiency bound, by drawing on the semiparametric inference literature \citep{Robins-1994, Chen2000AUA}, and prior PPI efficiency improvement proposals \citep{Gronsbell2024AnotherLA}. They propose a cross-fitted approach that fits a recalibration function, and applies rectification and power tuning to different folds\footnote{In the general M-estimation setting, their algorithm requires a three-way split. In mean estimation, a two-way split suffices.}. We build on their works by exploiting both power-tuning and recalibration, but exploiting the clear separation between our tasks, yielding natural data splitting schemes. In concurrent work to ours, \cite{calibeating} introduce the idea of calibrating the AI predictions on the \emph{outcome scale} to "calibeat" PPI. This work is related to \cite{RePPI-2025} but differs in two ways. \emph{Firstly}, they focus on recalibration only on the outcome scale, that is, the recalibrator is a map from $\hat{Y} \mapsto Y$ and not $X, \hat{Y} \mapsto Y$. We also follow this approach, but learn the recalibrator on auxiliary data; which is crucial in our sparse data regime with related tasks.  \emph{Secondly}, \cite{calibeating} differs from \cite{RePPI-2025} and our work in that they emphasize asymptotic theory without data splitting. Thus, \cite{calibeating}'s work is similar in approach (recalibrating on the outcome scale only and using isotonic regression) as well as some specific results (recognising the importance of nonlinear recalibration to improve on PPI++\footnote{In their work, recognising that linear calibration is first-order equivalent to \ppipp}), but different in setting (multi-task vs. single-task) and resulting statistical procedures. We note that in our small data regime, recalibration within the task of interest has not proven to be an effective choice, and cross-task borrowing is crucial.

\paragraph{Survey Sampling} While in the conception of this work, we were inspired by the rapidly growing literature on prediction powered inference, the working model of finite population sampling that reflects how ML benchmarks can be evaluated using surrogate information closely resembles that of design-based model-assisted survey sampling.\footnote{We note here that the recalibration insights that we propose are applicable to either sampling regime (c.f. Appendix~\ref{appendix:superpopulation}), with slightly adjusted power-tuning parameters and different variance estimators for CI implementation.} Survey Sampling \citep{Sarndal2003ModelAS} is concerned with drawing conclusions about a finite population using ground-truth data observed on a subset of said population. \cite{Mozer2026PPIIT} and \cite{kluger} recently remarked the connections between survey sampling PPI estimators; in particular, for the special case of mean estimation, PPI is algebraically equivalent to the difference estimator dating back to at least \citep{Cassel1976SomeRO} and PPI++ takes the form of the GREG estimator \citep{Cassel1976SomeRO, sarndal_design_vs_model} covered in depth in the book \cite{Sarndal2003ModelAS}. While this literature has not been concerned with using AI predictions, much of the statistical machinery is shared.

\paragraph{Small Area Estimation}  Small area estimation (SAE) is concerned with making inferences and predictions about subdomains of a population when labeled data for such sub-populations are scarce, typically in geographical settings \citep{raomolinabook}. While our setting is different in instantiation of each component (the nature of the auxiliary data, the nature of the sub-domains / tasks, and the choice of globally learned model to boost power), SAE and our work share their primary motivation: using data from other domains to improve inference within the domain of interest. Approaches that aim to "borrow strength" (also called \emph{indirect} methods \citep{lehtonenchapter}) from other domains can be categorized in either model-based, or design-based. Model-based approaches use an explicit model to generalize inferences from other domains to the domain of interest \citep{Fay1979EstimatesOI, raomolinabook}. These methods usually have lower variance, but when the model is misspecified, they may be biased. These ideas have been explored in applications to AI evaluation \citep{fogliato2024precise}. On the other hand, design-based approches aim to be unbiased or at least consistent, regardless of the misspecification of the working model, as is our aim here. Standard approaches such as indirect GREG differ from our work. Indeed, an indirect GREG estimator would amount to using global information for the power-tuning stage, with no recalibration step (c.f. \cite{lehtonenchapter}, Section 4.1.1.). Our usage of indirect information is different, partially because our auxiliary data is already a highly correlated AI prediction of the target variable: we use local power-tuning (corresponding to a direct GREG estimator), but use a non-parametric recalibrator learned from indirect (global) information. 
Finally, we note that approaches exist that can be situated in between the design- and model-based methodologies; for instance \cite{betsy} take a conformal prediction approach to SAE that always maintains frequentist coverage, and is volume-optimal when the working model inducing the conformity measure is well-specified (however, they do not make use surrogate information or predictions).

\newpage

\bibliographystyle{plainnat}
\bibliography{clean_references}

\appendix

\newpage
\section{Missing Derivations and Proofs}
\label{appendix:proofs}
\subsection{Derivation of Oracle Variance $V(s,\lambda)$}
\label{appendix:oracle-variance}

We derive the variance functional used in \Cref{sec:power-tuning}.  Fix a recalibration map $s$ and write the shorthand $s_i := s(X_i)$. Let $\lambda \in \R$ be the scalar power-tuning coefficient.
Using the shorthands
$$
\qquad
\bar Y_\labeled := \frac{1}{n}\sum_{i\in\labeled}Y_i,
\quad
\bar s_\labeled := \frac{1}{n}\sum_{i\in\labeled} s_i,
\quad
\bar s_N := \frac{1}{N}\sum_{i \in [N]} s_i,
$$
the power-tuned estimator from~\eqref{eq:fp-estimator} reads as
$$
\thetahat_\lambda
=
\bar Y_\labeled + \lambda(\bar s_N - \bar s_\labeled),
$$
and the target parameter is the finite-population mean $\thetastar = \bar Y_N = N^{-1}\sum_{i \in [N]} Y_i$. 

\xhdr{Characterizing Estimation Error} 
Subtracting $\thetastar = \bar Y_N$ from $\thetahat_\lambda$ and grouping the labeled and full-population contributions,
\begin{equation*}
\thetahat_\lambda - \thetastar
\;=\; \bigl(\bar Y_\labeled - \lambda\bar s_\labeled\bigr) - \bigl(\bar Y_N - \lambda\bar s_N\bigr)
\;=\; \frac{1}{n}\sum_{i\in\labeled}(Y_i-\lambda s_i) - \frac{1}{N}\sum_{i \in [N]}(Y_i-\lambda s_i).
\end{equation*}
In other words, the estimation error is proportional to how accurately the labeled samples of $Y - \lambda s$ reflect the full finite-population. Let us denote by $Z_i := Y_i - \lambda s_i$. Let us take the expectation $\E$ and variance $\Var$ operators to be over the randomness in the draw of the subset $\labeled$. Clearly, $\E[\bar s_\labeled] = \bar s_N$ so that for \emph{fixed} $\lambda$ the estimator $\hat{\theta}_\lambda$ is unbiased in finite-samples, with its variance given by
\begin{equation*}
    V(s,\lambda) := \Var[\thetahat_\lambda] = \E[(\hat{\theta}_\lambda - \thetastar)^2] = \E\left[\left(\frac{1}{n}\sum_{i \in \labeled} Z_i - \frac{1}{N}\sum_{i \in [N]} Z_i\right)^2\right] = \Var[\Bar{Z}_{\labeled}].
\end{equation*}
This can be seen as the Horvitz-Thompson \citep{Horvitz1952AGO, Narain1951} estimator for the population mean of the random variables $Z$ under simple random sampling without replacement. 

\xhdr{Estimator Variance and FPC} Appealing to classical results on the Horvitz-Thompson estimator, (see, e.g. \cite{kwangsurveystatsbook}, Chapters 2 and 3) we can express
\begin{equation*}
    \Var[\Bar{Z}_{\labeled}] = \frac{1}{n}\left(1-\frac{n-1}{N-1}\right) \sigma_{Z}^2,
\end{equation*}
where $\sigma_{Z}^2 = \frac{1}{N}\sum_{i \in [N]} (Z_i - \bar{Z})^2$. This can also be rewritten as
\begin{equation*}
    \Var[\hat{Z}] = \frac{1}{n}\left(1-\frac{n}{N}\right) S_{Z}^2 \quad \text{ where } \quad S_{Z}^2 = \frac{1}{N-1}\sum_{i \in [N]} (Z_i - \bar{Z})^2
\end{equation*}
is the sample variance. We will approximate the latter with the plug-in estimate
\begin{equation*}
    \hat{S}_{Z}^2 = \frac{1}{n-1} \sum_{i \in \labeled} (Z_i - \bar{Z}_{\labeled})^2,
\end{equation*}
that satisfies $S_{Z}^2 = \E[\hat{S}_{Z}^2]$ under simple random sampling (e.g., \citep{kwangsurveystatsbook}, Chapter 3) to build confidence intervals (c.f.\ \ref{app:implementation-details}). 
 Recalling that we defined $Z = Y - \lambda s$, we can expand
\begin{align}
V(s,\lambda) = \Var[\Bar{Z}_{\labeled}]
&= \frac{1}{n}\left(1-\frac{n-1}{N-1}\right) \sigma_{Z}^2 \nonumber \\
&= \frac{1}{n}\left(1-\frac{n-1}{N-1}\right)\bigl[\sigma_Y^2 - 2\lambda\,\Cov_N(Y,s) + \lambda^2\,\Var_N(s)\bigr]. \label{eq:Vfp-app}
\end{align}
\xhdr{Optimal $\lambda$ and oracle variance} It is clear then that the optimal $\lambda$ should minimize the quantity $\sigma_Z^2$ (or equivalently $S_Z^2$).  The right-hand side of~\eqref{eq:Vfp-app} is quadratic in $\lambda$, with first-order condition $-2\Cov_N(Y,s) + 2\lambda\Var_N(s) = 0$. The minimizer is therefore
\[
\lambda^\star(s) \;=\; \frac{\Cov_N(Y,s)}{\Var_N(s)}.
\]
Substituting back and writing $\rho_N^2(Y,s) := \Cov_N(Y,s)^2 / (\Var_N(Y)\,\Var_N(s))$ for the finite-population squared correlation, we get
\[
V^\star(s)
    =
   \frac{1}{n} \left(1-\frac{n}{N}\right)
    S_Y^2 \bigl(1-\rho_N^2(Y, s)\bigr).
\]

\subsection{Derivations for the Superpopulation Case}
\label{appendix:superpopulation}
The variance expression in Appendix~\ref{appendix:oracle-variance} differs from the standard superpopulation
PPI++ expression because, in our finite-population sampling model, the
surrogate mean $\bar s_N$ is observed exactly. In the superpopulation model, this
is no longer true. Suppose that we observe an i.i.d.\ labeled sample
$$
(X_i,Y_i,\widehat Y_i,O_i)_{i=1}^n,
\qquad O_i=1,
$$
and an \emph{independent} i.i.d.\ unlabeled sample
$$
(X'_j,Y'_j,\widehat Y'_j,O'_j)_{j=1}^N,
\qquad O'_j=0,
$$
drawn from the same superpopulation. The outcomes $Y'_j$ in the unlabeled sample
are unobserved, while the surrogate predictions $\widehat Y'_j$ are observed. For
a fixed recalibration map $s$, write
$$
s_i=s(\widehat Y_i),
\qquad
s'_j=s(\widehat Y'_j).
$$
Consider the estimator
$$
\widehat\theta_\lambda
=
\bar Y_n+\lambda(\bar s'_N-\bar s_n),
$$
where
$$
\bar Y_n=\frac{1}{n}\sum_{i=1}^n Y_i,
\qquad
\bar s_n=\frac{1}{n}\sum_{i=1}^n s_i,
\qquad
\bar s'_N=\frac{1}{N}\sum_{j=1}^N s'_j.
$$
Since the labeled and unlabeled samples are independent,
$$
\Var(\widehat\theta_\lambda)
=
\frac{1}{n}\Var(Y-\lambda s)
+
\frac{\lambda^2}{N}\Var(s).
$$
Equivalently,
$$
V_{\rm sup}(s,\lambda)
=
\frac{1}{n}\Var(Y)
-
\frac{2\lambda}{n}\Cov(Y,s)
+
\lambda^2\left(\frac{1}{n}+\frac{1}{N}\right)\Var(s).
$$
Minimizing this quadratic in $\lambda$ gives
$$
\lambda^\star_{\rm sup}(s)
=
\frac{N}{n+N}\,
\frac{\Cov(Y,s)}{\Var(s)}.
$$
Plugging this value back into the variance formula yields
$$
V^\star_{\rm sup}(s)
=
\frac{1}{n}\Var(Y)
-
\frac{1}{n} \left(\frac{N}{n+N}\right)\frac{\Cov(Y,s)^2}
{\Var(s)}.
$$
Using
$$
\rho^2(Y,s)
=
\frac{\Cov(Y,s)^2}{\Var(Y)\Var(s)},
$$
this simplifies to
$$
V^\star_{\rm sup}(s)
=
\frac{\Var(Y)}{n}
\left(
1-\frac{N}{n+N}\rho^2(Y,s)
\right).
$$
Note the differences with $V^\star$ in Appendix~\ref{appendix:oracle-variance}: $V^\star$ includes a finite population correction (FPC) that makes the variance go to zero when $n$ approaches $N$ even in finite samples. This makes sense, since when $n=N$, in the finite-population setting, there is no randomness left. On the other hand, in the superpopulation case, when $N$ is not much larger than $n$, variance reduction has \emph{less} of an impact, since the correlation term gets downweighted. This tracks our intuition: if $N$ is not particularly large, then estimating $\bar s'_N$ introduces variance too. 

 See \cite{PPIpp-2023}, Example 6.1.\ for the original result in the superpopulation case. We note that these formulas allow easy implementation of our method in the superpopulation setting. What changes is the construction of the confidence intervals, as well as the power-tuning parameters. However, cross-task recalibration works just the same.

\subsection{Proof of Corollary~\ref{cor:nonlinear} (Nonlinear Necessity)}
\label{appendix:nonlinear-necessity}
We work in the finite-population setting of \Cref{appendix:oracle-variance}. By~\eqref{eq:Vstar},
\begin{equation}\label{eq:vstar-recall}
V^\star(s) = \frac{1}{n}\left(1-\frac{n}{N}\right)S_Y^2\bigl(1-\rho_N^2(Y, s)\bigr),
\end{equation}
so $V^\star$ depends on $s$ only through $\rho_N^2(Y,s)$. Throughout this proof we restrict to surrogates of the form $s_i = \phi(\hat{Y}_i)$\footnote{i.e. we do not take into account covariate information.} for a measurable $\phi : \R \to \R$, and write $V^\star(\phi)$ for the corresponding optimal oracle variance. Minimizing $\phi \mapsto V^\star(\phi)$ is then equivalent to maximizing $\rho_N^2(Y, \phi(\hat{Y}))$ over $\phi$.

\xhdr{Finite-population conditional expectation}
Let $\mathcal{Z} := \{\hat{Y}_i : i\in[N]\}$ denote the population support of
the proxy/surrogate and, for each $z\in\mathcal{Z}$, let
$I_z := \{i\in[N] : \hat{Y}_i=z\}$. Define
\begin{equation}
    m(z) := \E_N[Y_i \sep \hat{Y}_i = z] = \frac{1}{|I_z|}\sum_{i\in I_z} Y_i,\qquad z\in\mathcal{Z}.
\end{equation}
By construction, $\sum_{i\in I_z}\bigl(Y_i - m(z)\bigr) = 0$ for every
$z\in\mathcal{Z}$, and hence $\overline{m(\hat{Y})}_N = \bar Y_N$.

\begin{lemma}[Residual Orthogonality]\label{lem:fp-orth}
For every measurable $\phi:\R\to\R$,
\[
\Cov_N\bigl(Y - m(\hat{Y}),\; \phi(\hat{Y})\bigr) = 0,
\]
and consequently $\Cov_N(Y, \phi(\hat{Y})) = \Cov_N(m(\hat{Y}), \phi(\hat{Y}))$.
\end{lemma}

\begin{proof}
Write $c := \overline{\phi(\hat{Y})}_N$. Since $Y - m(\hat{Y})$ is centered,
\begin{align*}
\Cov_N\bigl(Y - m(\hat{Y}),\, \phi(\hat{Y})\bigr)
&= \frac{1}{N}\sum_{i \in [N]} \bigl(Y_i - m(\hat{Y}_i)\bigr)\bigl(\phi(\hat{Y}_i) - c\bigr)\\
&= \frac{1}{N}\sum_{z \in \mathcal{Z}} \sum_{i \in I_z} \bigl(Y_i - m(\hat{Y}_i)\bigr)\bigl(\phi(\hat{Y}_i) - c\bigr)\\
&= \frac{1}{N}\sum_{z \in \mathcal{Z}} \sum_{i \in I_z} \bigl(Y_i - m(z)\bigr)\bigl(\phi(z) - c\bigr)\\
&= \frac{1}{N}\sum_{z\in\mathcal{Z}} \bigl(\phi(z) - c\bigr)\sum_{i\in I_z}\bigl(Y_i - m(z)\bigr)
\;=\; 0,
\end{align*}
using $\sum_{i\in I_z}(Y_i - m(z)) = 0$ for every $z$. \qedhere
\end{proof}

We are now in a position to prove the corollary.

\begin{proof}[Proof of Corollary~\ref{cor:nonlinear}]
By Lemma~\ref{lem:fp-orth} and the covariance inequality
\begin{equation}\label{eq:cs-bound}
\Cov_N(Y, \phi(\hat{Y}))^2
= \Cov_N\bigl(m(\hat{Y}), \phi(\hat{Y})\bigr)^2
\;\le\;
\Var_N\bigl(m(\hat{Y})\bigr)\,\Var_N\bigl(\phi(\hat{Y})\bigr),
\end{equation}
with equality if and only if there exist $a,b\in\R$\footnote{This follows from Cauchy-Schwarz on the centered versions of these random variables.} such that
\begin{equation}\label{eq:affine-eq}
\phi(\hat{Y}_i) \;=\; a\, m(\hat{Y}_i) + b
\qquad\text{for every } i\in[N].
\end{equation}
(When both centered vectors are nonzero, equality forces $a\ne 0$; the
degenerate cases $\Var_N(\phi(\hat{Y}))=0$ or $\Var_N(m(\hat{Y}))=0$ make both sides
of~\eqref{eq:cs-bound} vanish and the bound trivially holds.) Dividing~\eqref{eq:cs-bound}
by $\Var_N(Y)\,\Var_N(\phi(\hat{Y}))$ and using
$R^2_{Y\sim \hat{Y}} = \Var_N(m(\hat{Y}))/\Var_N(Y)$ gives
\begin{equation}\label{eq:rho-bound}
\rho_N^2\bigl(Y, \phi(\hat{Y})\bigr) \;\le\; R^2_{Y\sim \hat{Y}},
\end{equation}
with equality under~\eqref{eq:affine-eq}. The choice $\phi^\star(z) = m(z)$ (taking $a=1, b=0$) attains the bound, so by~\eqref{eq:vstar-recall},
\begin{equation}\label{eq:opt-vstar}
\inf_\phi V^\star(\phi)
= \frac{1}{n}\left(1-\frac{n}{N}\right)S_Y^2\bigl(1 - R^2_{Y\sim \hat{Y}}\bigr),
\end{equation}
attained at $\phi = m$.

\xhdr{Equivalence with nonlinearity of $m$}
Specializing~\eqref{eq:rho-bound} to $\phi = \mathrm{id}$ yields $\rho_N^2(Y, \hat{Y}) \le R^2_{Y\sim \hat{Y}}$, with equality if and only if there exist $a,b\in\R$ such that $\hat{Y}_i = a\,m(\hat{Y}_i) + b$ for every $i\in[N]$. Equivalently, $m$ is affine on the population support $\mathcal{Z}$. Therefore strict improvement $\inf_\phi V^\star(\phi) < V^\star(\mathrm{id})$ is achievable if and only if $m$ is \emph{not} affine on $\mathcal{Z}$, which is the nonlinearity condition of the corollary.

\xhdr{Maximum achievable gain}
Combining~\eqref{eq:vstar-recall} for the identity recalibration
with~\eqref{eq:opt-vstar},
\[
V^\star(\mathrm{id}) \;-\; \inf_\phi V^\star(\phi)
= \frac{1}{n}\left(1-\frac{n}{N}\right)S_Y^2
   \bigl(R^2_{Y\sim \hat{Y}} - \rho_N^2(Y, \hat{Y})\bigr),
\]
which matches~\eqref{eq:max-gain} and completes the proof. \qedhere
\end{proof}

\newpage
\section{Implementation Details}
\label{app:implementation-details}

\subsection{\reppi (mean-estimation specialization)}
\label{appendix:cis:reppi}

\citet[Algorithm~1]{RePPI-2025} prescribe a three-way split of the labeled set into three folds: the first is used to compute an initial estimator $\hat\theta_0$ that pins down the linearization point for the conditional score; the second is used to fit the recalibrator $\hat s$; and the third together with the unlabeled data is used to compute the optimal control-variate matrix $\hat M$ (i.e. $\hat{\lambda}$ in one dimension) \emph{and} the rectified estimator. The procedure is then rotated over the three folds and the per-rotation estimators aggregated. The first fold is needed in the general $Z$-estimation setting because the optimal imputed loss $s^\star = \E[\nabla \ell_{\thetastar}(X,Y)\sep X, \hat Y]$ depends on $\thetastar$, which is unknown and must be replaced by an initial estimate $\hat\theta_0$. However, in the mean-estimation case considered in this paper, the conditional score factorizes as
\(
    \E[\nabla_\theta \ell_\theta(X,Y) \sep X, \hat Y] = \theta - \E[Y \sep X, \hat Y],
\)
so the dependence on $\theta$ is purely additive and can be isolated to give the estimator in closed form. Therefore, the recalibration target $\E[Y \sep \hat Y^, X]$ (in our case $\E[Y \sep \hat Y]$, since we do not model dependence on the underlying covariates) does not depend on any initial estimator. The fold reserved for $\hat\theta_0$ in \citet[Alg.~1]{RePPI-2025} can therefore be dropped, and the three-way split collapses to a two-way split $\labeled^\taskindex = A \cup B$. Another difference in our implementation is the finite-population nature of our task. In their paper, they assume independence of the labeled and unlabeled data. For us, the labeled data is a subset of the data on which the ML predictions get evaluated. Therefore, in order to not fit the recalibrator on any data on which it is evaluated in the final PPI estimator, we instead "stitch" together cross-fitted predictions. In practice, the difference from this adaptation was, however, seen to be negligible (indeed, the only (minor) difference between the resulting estimators is in how labeled points are handled in the sum ranging over $i \in [N]$). The exact procedure is given in Algorithm~\ref{alg:reppi-2fold}.

\begin{algorithm}[h!]
\caption{\reppi (\cite{RePPI-2025}, finite-population, mean-estimation adaptation)}
\label{alg:reppi-2fold}
\begin{algorithmic}[1]
\Require Task dataset $\dataset^\taskindex = (\hat Y_i^\taskindex, Y_i^\taskindex, O_i^\taskindex)_{i=1}^N$, recalibration class $\mathcal H$.
\State Randomly split $\labeled^\taskindex$ evenly into two folds $A$ and $B$.
\State  For each $F \in \{A, B\}$, fit $\hat s_F \in \mathcal H$ on $\{(\hat Y_i^\taskindex, Y_i^\taskindex)\}_{i \in F}$ so that $\hat s_F(\hat Y^\taskindex) \approx \E_N[Y^\taskindex \sep \hat Y^\taskindex]$.\Comment{collapses Steps 2--3 of \citet[Alg.~1]{RePPI-2025}.}
\State  Define the stitched cross-fitted predictions on the labeled set,
$w_i^\taskindex = \hat s_B(\hat Y_i^\taskindex)$ for $i \in A$ and $w_i^\taskindex = \hat s_A(\hat Y_i^\taskindex)$ for $i \in B$,
and compute the scalar power-tuning factor (analogue of $\hat M$ in Eq.~(9) of \citet{RePPI-2025})
$$
    \hat\lambda_{\labeled^\taskindex}
    \;=\;
    \frac{\Cov_{\labeled^\taskindex}(Y^\taskindex,\, w^\taskindex)}{\Var_{\labeled^\taskindex}(w^\taskindex)}.
$$
\State  On the unlabeled data, define the ensemble prediction $w_i^\taskindex = \tfrac12\!\left(\hat s_A(\hat Y_i^\taskindex) + \hat s_B(\hat Y_i^\taskindex)\right)$ for $i \in [N] \setminus \labeled^\taskindex$.
\State Return the rectified estimator
$$
    \hat\theta^\taskindex
    \;=\;
    \frac{1}{n}\!\sum_{i\in\labeled^\taskindex}\! Y_i^\taskindex
    \;-\;
    \hat\lambda_{\labeled^\taskindex}
    \!\left(
    \frac{1}{n}\!\sum_{i\in\labeled^\taskindex}\! w_i^\taskindex
    - \frac{1}{N}\!\sum_{i \in [N]}\! w_i^\taskindex
    \right)
$$
\end{algorithmic}
\end{algorithm}

\subsection{Confidence Intervals}
\label{appendix:cis:wald}
We collect here the per-method recipes used to form the studentized Wald confidence intervals reported in Section~\ref{sec:casestudy} and Appendix~\ref{app:ablation}. Throughout, $\alpha$ is the prescribed miscoverage level, $t_{n-1, 1-\alpha/2}$ is the $1-\alpha/2$ quantile of a Student-$t$ distribution with $n-1$ degrees of freedom, and each method outputs a point estimate $\hat\theta^\taskindex$ together with a residual sequence $(r_i^\taskindex)_{i \in \labeled^\taskindex}$ (to be defined below).
All methods share the same skeleton, and differ only in how the residual sequence $r_i^\taskindex$ used for the variance plug-in is constructed. Given the residuals, we build
\begin{equation}\label{eq:wald-skeleton}
    \widehat{V}(\hat\theta^\taskindex)
    \;=\;
    \frac{1}{n}\left(1 - \frac{n}{N}\right) \hat S_{r}^2,
    \qquad
    \hat S_{r}^2
    \;=\;
    \frac{1}{n - 1} \sum_{i \in \labeled^\taskindex} \big(r_i^\taskindex - \bar r^\taskindex\big)^2,
\end{equation}
and report the $1-\alpha$ Wald CI
\begin{equation}\label{eq:wald-ci}
    \mathrm{CI}_{1-\alpha}^\taskindex
    \;=\;
    \hat\theta^\taskindex \pm t_{n-1,\, 1-\alpha/2}\,\sqrt{\widehat{V}(\hat\theta^\taskindex)}.
\end{equation}
The form of $\widehat{V}$ in \eqref{eq:wald-skeleton} is the unbiased SRS-without-replacement variance estimator with a finite-population correction (c.f. Appendix~\ref{appendix:oracle-variance} and \citep{kwangsurveystatsbook}, Chapter~3), applied to the residual sequence specified per method below. All power tuning coefficients $\hat{\lambda}$ are clipped to $[0,1]$ in our implementation.

The residual $r_i^\taskindex$ for each method is (with $\hat\lambda_{\labeled^\taskindex}$ being either the local power tuning coefficient, or a constant $1$, depending on the plot in question):
\begin{itemize}
    \item \classical: $r_i^\taskindex = Y_i^\taskindex$.
    \item \ppi: $r_i^\taskindex = Y_i^\taskindex - \hat Y_i^\taskindex$ (c.f. \citep{PPI-2023}) 
    \item \ppipp: $r_i^\taskindex = Y_i^\taskindex - \hat\lambda_{\labeled^\taskindex}\, \hat Y_i^\taskindex$ (c.f. \cite{PPIpp-2023}).
    \item \reppi: $r_i^\taskindex = Y_i^\taskindex - \hat\lambda_{\labeled^\taskindex} w_i^\taskindex$ (c.f. Algorithm~\ref{alg:reppi-2fold}, our baseline adaptation of \cite{RePPI-2025}).
    \item \greppi: $r_i^\taskindex = Y_i^\taskindex - \hat\lambda_{\labeled^\taskindex}\, \hat s^\excepttask(\hat Y_i^\taskindex)$ (c.f. Algorithm~\ref{alg:loo-reppi})
    \item \areppi: $r_i^\taskindex = Y_i^\taskindex - \hat\lambda^{\mathrm{oof}}_{\labeled^\taskindex}\, u_i^\taskindex$ (c.f. Algorithm~\ref{alg:adaptive-reppi}).
\end{itemize}

A few remarks. The SE plug-ins above treat $\hat\lambda$ and $\hat s$ as fixed, resulting in bias in small samples, and we empirically diagnose its small-$n$ effect on coverage in Fig.~\ref{fig:main}~(a)--(d) when power tuning is used. We note in the experiments that power tuning is often not necessary if tasks are sufficiently homogenous or \areppi is used.

\section{Ablations and Synthetic Experiments}
\label{app:ablation}
We complement the real-data experiments with a controlled synthetic ablation designed to isolate the effect of task heterogeneity. The synthetic dataset is based on the family of monotone functions
\[
    f_p(x) = x^p, \qquad p \in [p_{\min}, 10],
\]
where changing $p_{\min} > 0$ controls the diversity of the auxiliary task population. Small values of $p_{\min}$ induce a highly heterogeneous collection of tasks: functions with $p \approx 0$ rise sharply near zero, while functions with large $p$ remain nearly flat over most of the domain and increase only close to one. As $p_{\min}$ increases, the task family becomes increasingly homogeneous, with most tasks sharing the same qualitative shape.

Figures~\ref{fig:main:synthetic_heterogeneous} and~\ref{fig:main:synthetic_homogeneous} summarize this ablation. In each figure, the top row visualizes the induced distribution of prediction--label relationships for each choice of $p_{\min}$, together with the (local and pooled) isotonic regression fits. When $p \in [1/10,10]$, the pooled fit (which will be very similar to the leave-one-out fit as well as the ) represents an average over substantially different monotone shapes and therefore does not match any single task particularly well. As the range narrows to $p \in [3,10]$ and further to $p \in [8,10]$, the task family concentrates around steep, high-curvature functions, and the pooled isotonic fit becomes much closer to the common task geometry.

The bottom rows report the MSE of \areppi, \greppi and \reppi with $\lambda=1$, normalized by the MSE of PPI, as a function of the number of labeled examples $n$. We make a few observations. We observe that the advantage of adaptive recalibration (\areppi) is most pronounced in the heterogeneous regimes. When $p_{\min}$ is small, auxiliary tasks vary widely in shape, so methods that can select or reweight useful auxiliary information improve steadily with $n$, while less adaptive aggregation remains bounded away from the best attainable performance (we note that estimation is still unbiased though). We note here that more complex ways to learn from auxiliary information may be viable. \areppi corresponds to a simple linear interpolation of a local and global recalibration function. Adaptively discovering which tasks are more similar may be of future interest. On the other hand, as the auxiliary population becomes more homogeneous, the gap between auxiliary-task methods changes: pooled or broadly shared information becomes less harmful because the auxiliary tasks are more mutually compatible, and the best methods approach very low relative MSE with fewer labeled examples. 

We conclude that when the auxiliary tasks are diverse, indiscriminately pooling them can introduce variance, because the predictions become worse for the target task. In this setting, methods that use (only) the labeled target data to identify relevant auxiliary structure achieve much larger gains (such as \reppi or \areppi, that achieves the same results adaptively). When the auxiliary tasks are already similar, the problem becomes easier: most auxiliary tasks provide useful information, and even simple transfer strategies perform well at very small sample sizes, the regime we are interested in here, and the ideal condition for \greppi.

\begin{figure}[h!]
    \centering

    \begin{subfigure}[t]{0.31\textwidth}
        \centering
        \includegraphics[width=\linewidth]{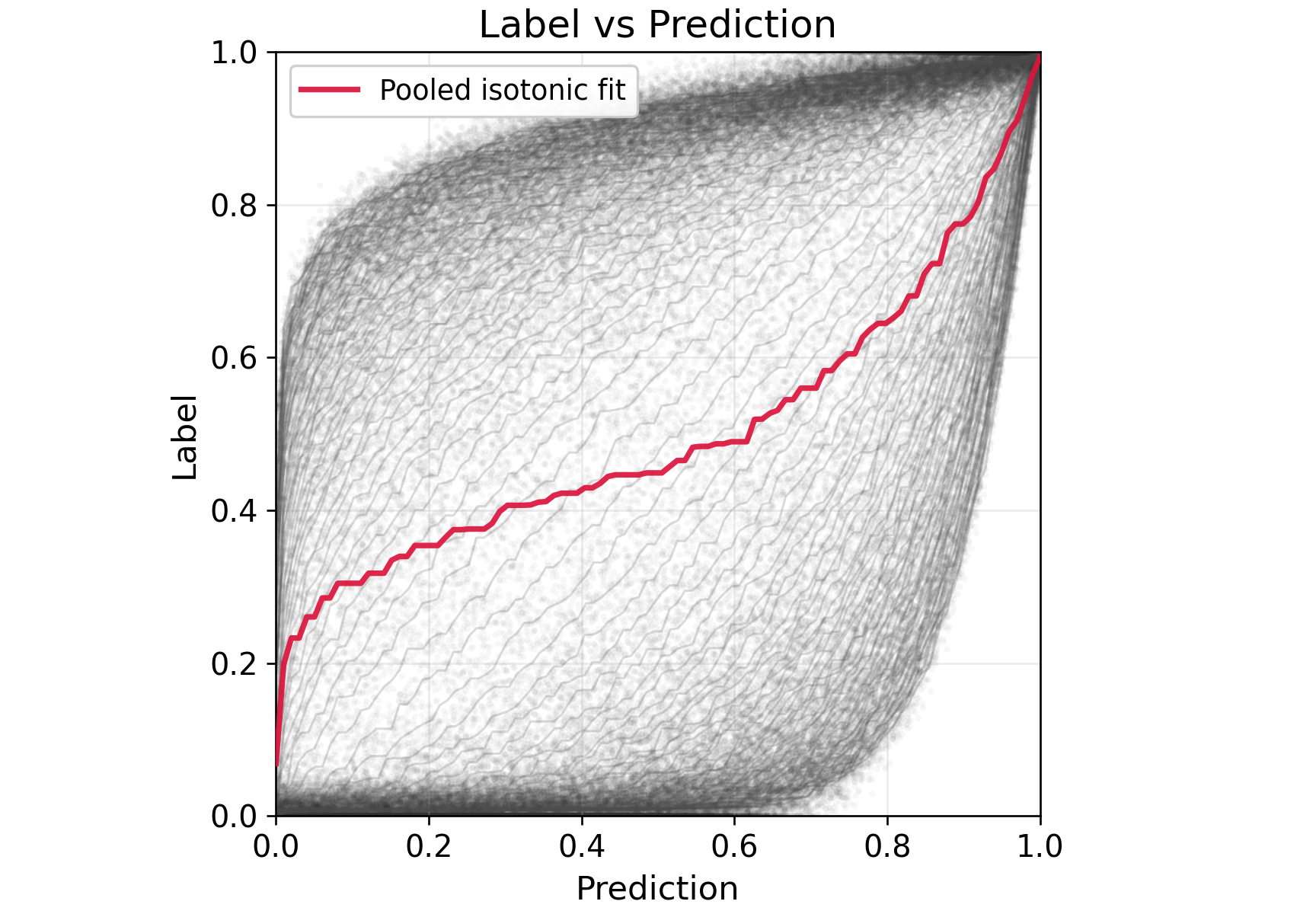}
        \caption{$p \in [1/10, 10]$}
        \label{fig:synthetic_pred_p0p1}
    \end{subfigure}
    \hfill
    \begin{subfigure}[t]{0.31\textwidth}
        \centering
        \includegraphics[width=\linewidth]{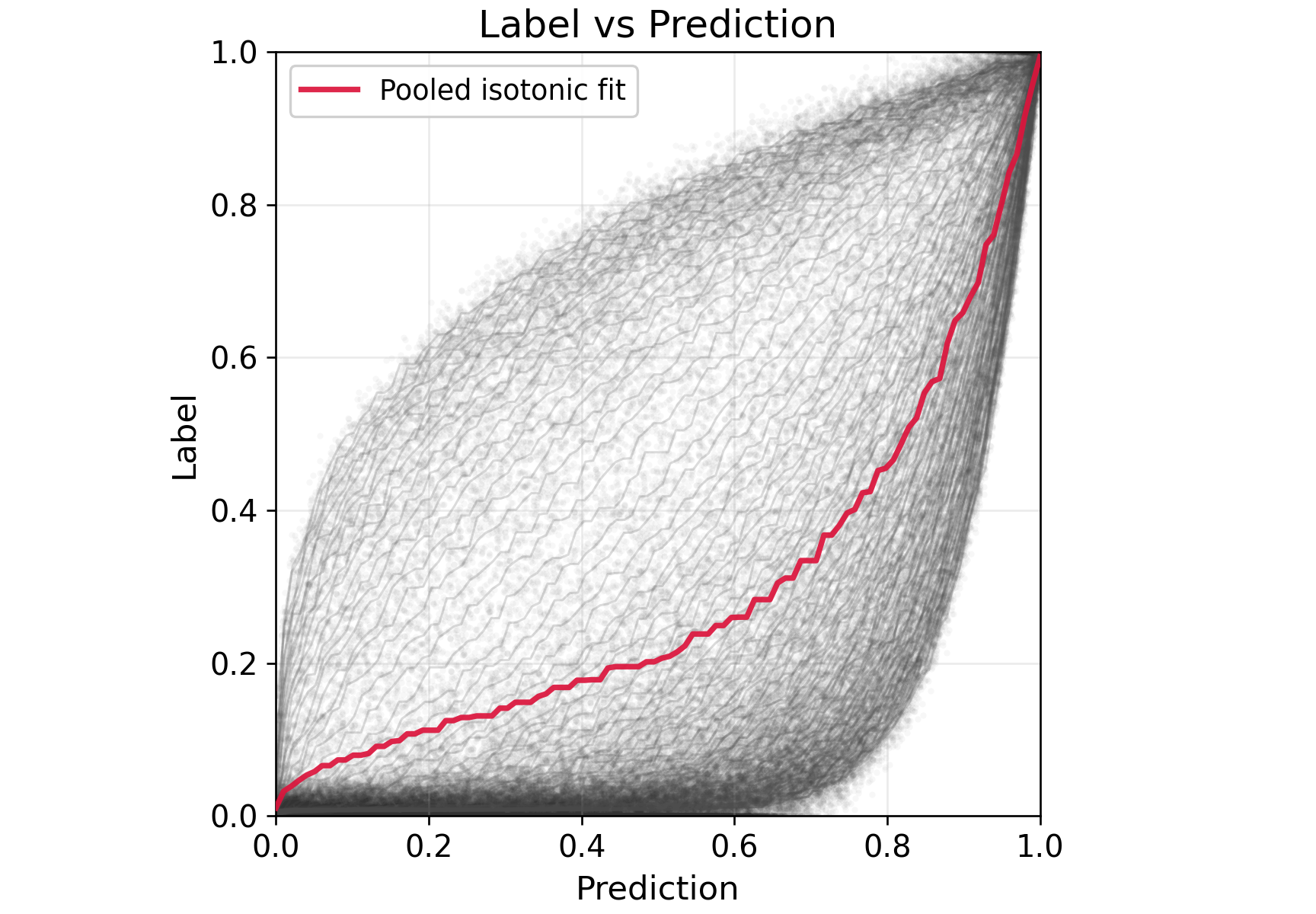}
        \caption{$p \in [3/10, 10]$}
        \label{fig:synthetic_pred_p0p3}
    \end{subfigure}
    \hfill
    \begin{subfigure}[t]{0.31\textwidth}
        \centering
        \includegraphics[width=\linewidth]{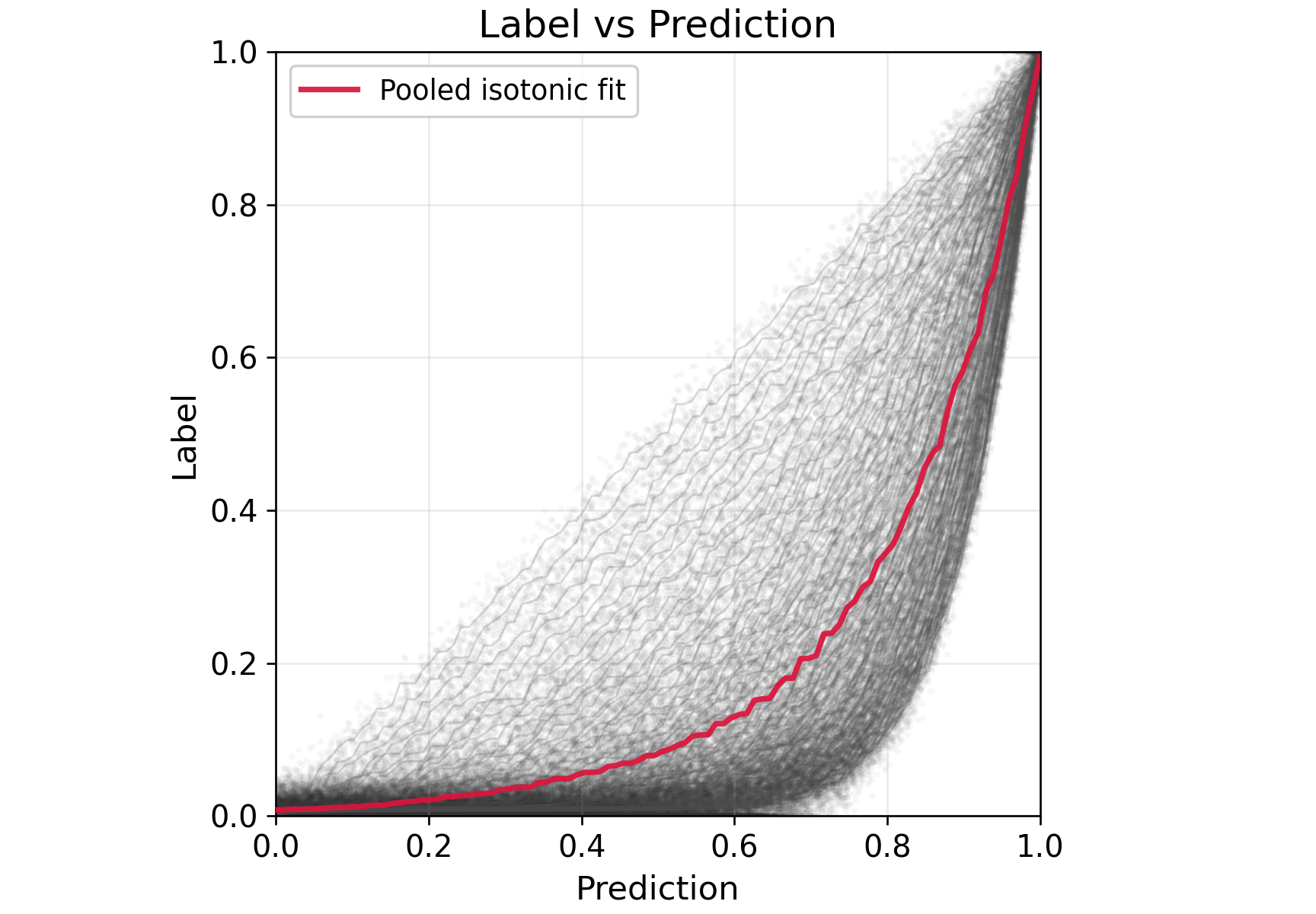}
        \caption{$p \in [1, 10]$}
        \label{fig:synthetic_pred_p1}

    \end{subfigure}

    \vspace{0.75em}

    \begin{subfigure}[t]{0.31\textwidth}
        \centering
        \includegraphics[width=\linewidth]{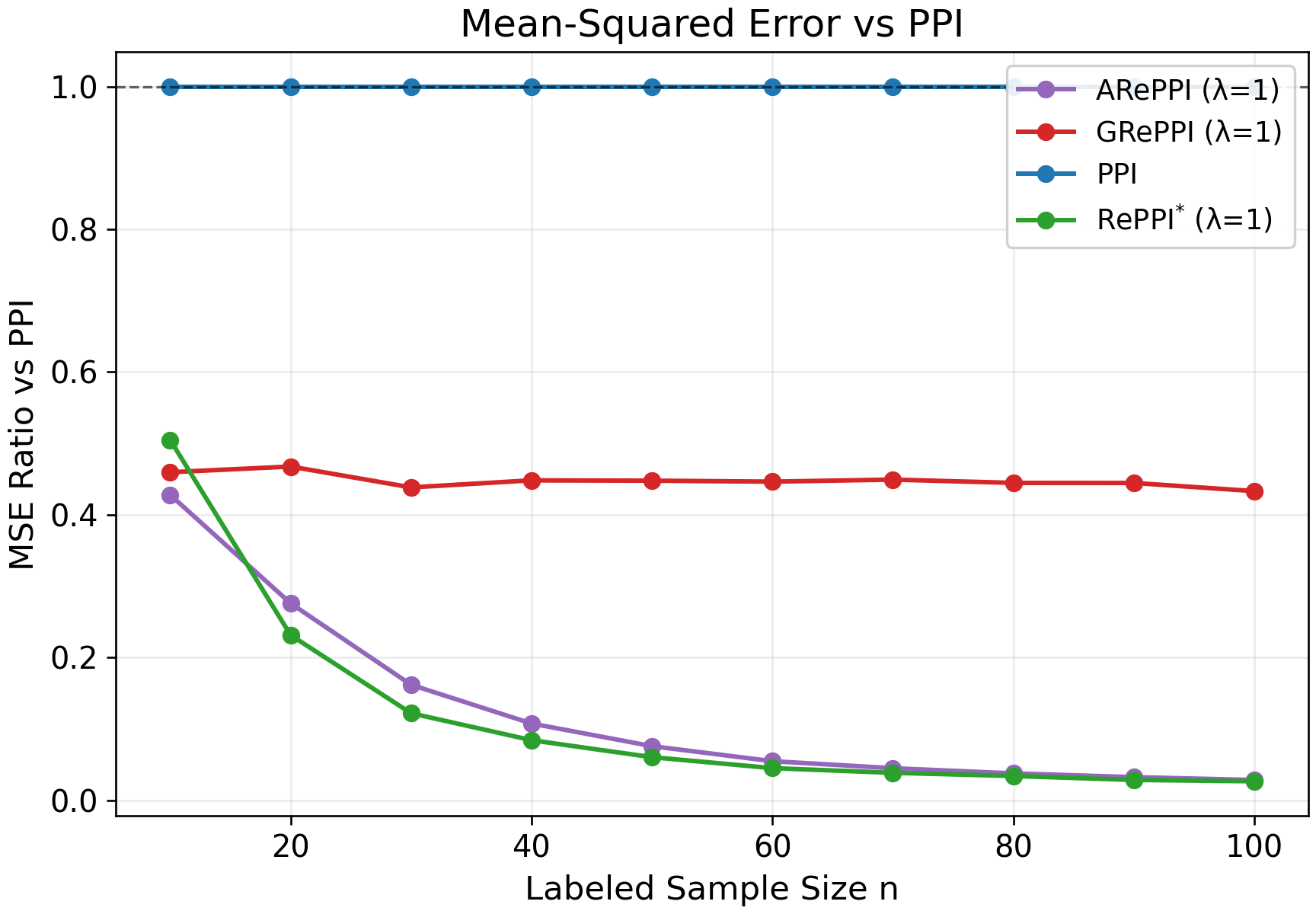}
        \caption{$p \in [1/10, 10]$}
        \label{fig:synthetic_mse_p0p1}
    \end{subfigure}
    \hfill
    \begin{subfigure}[t]{0.31\textwidth}
        \centering
        \includegraphics[width=\linewidth]{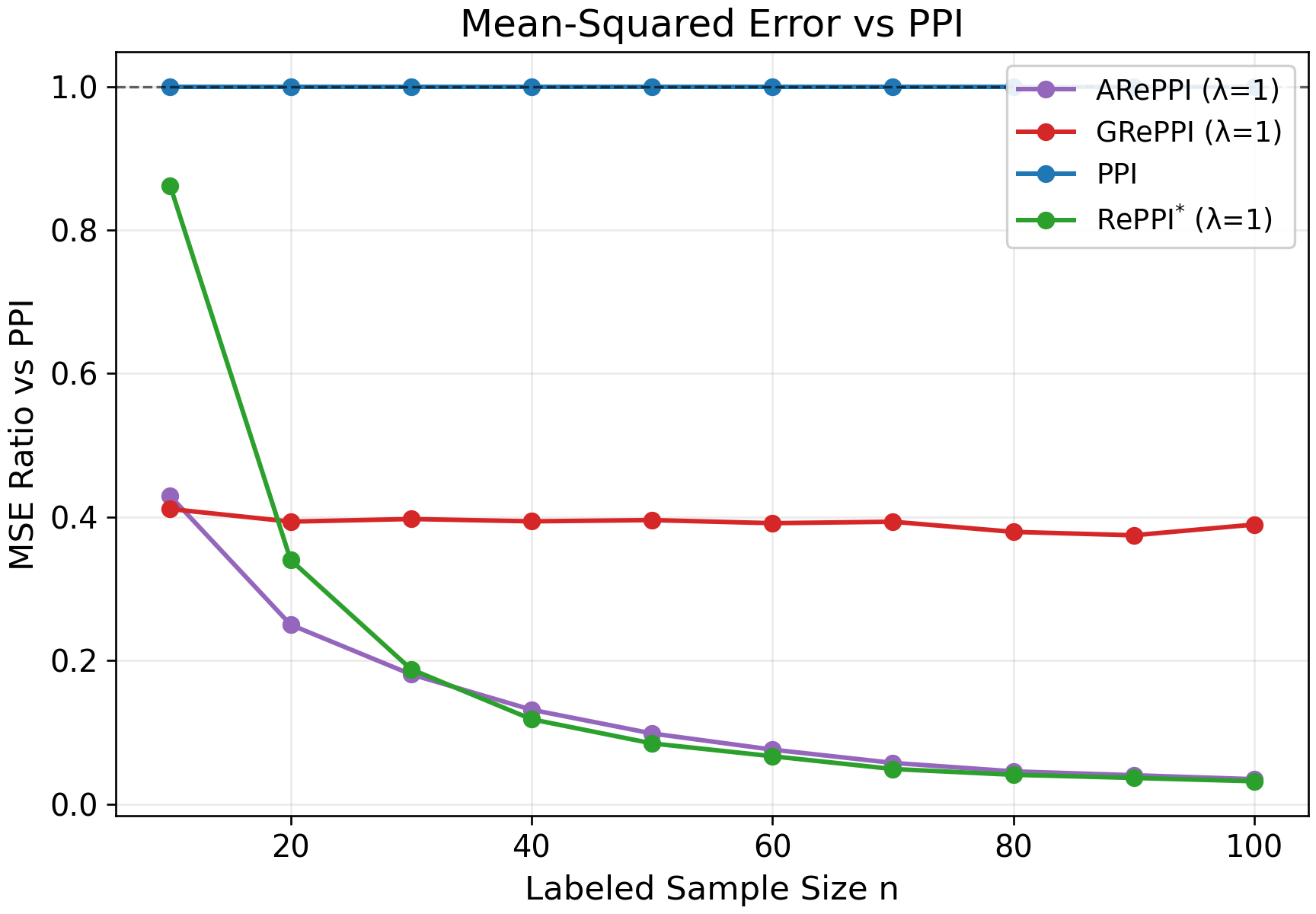}
        \caption{$p \in [3/10, 10]$}
        \label{fig:synthetic_mse_p0p3}
    \end{subfigure}
    \hfill
    \begin{subfigure}[t]{0.31\textwidth}
        \centering
        \includegraphics[width=\linewidth]{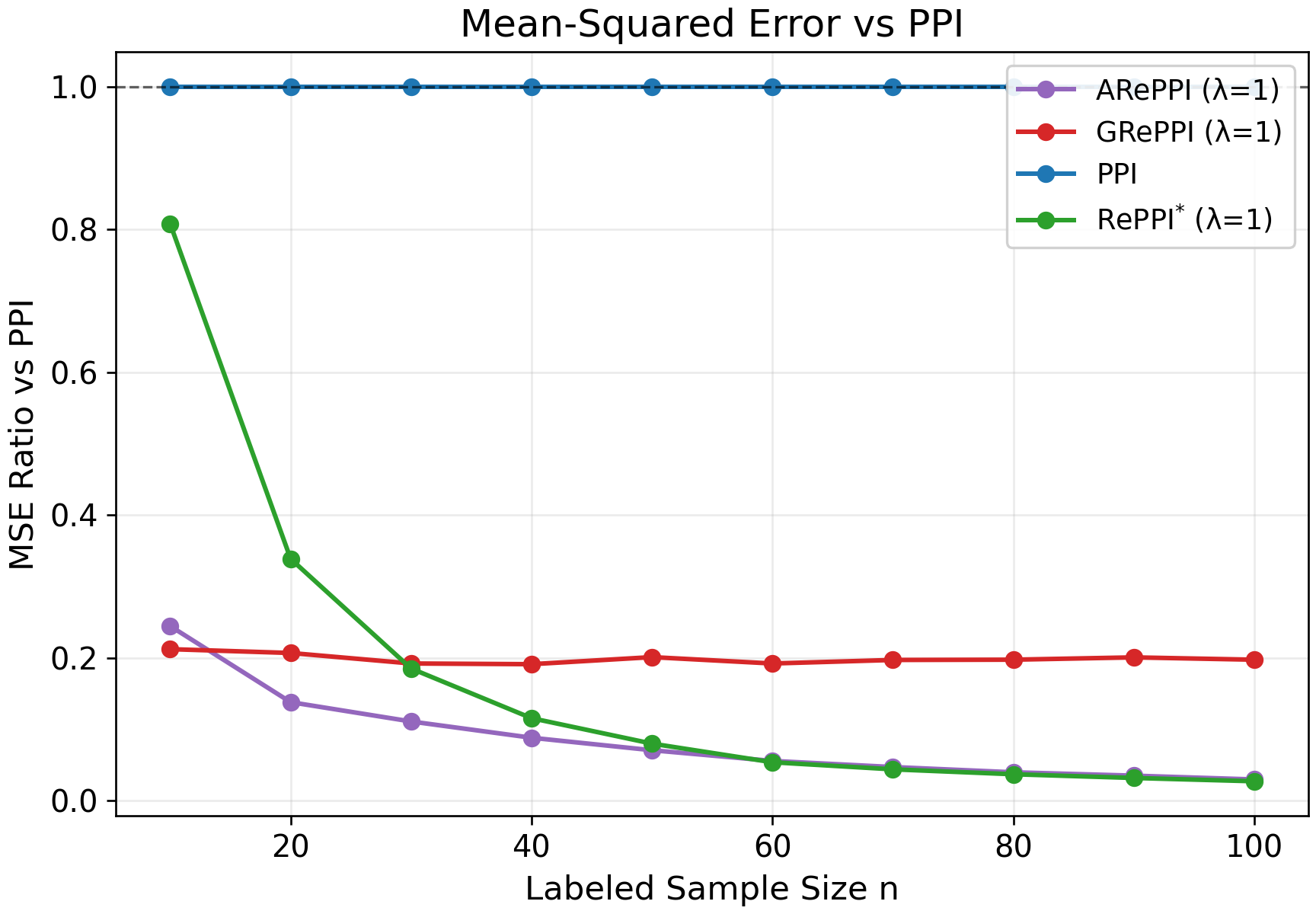}
        \caption{$p \in [1, 10]$}
        \label{fig:synthetic_mse_p1}
    \end{subfigure}

    \caption{
    Effect of auxiliary-task heterogeneity for the three broadest synthetic task families.
    Top row: prediction--label relationships induced by auxiliary tasks $f_p(x)=x^p$, together with the pooled isotonic fit.
    Bottom row: MSE normalized by the MSE of PPI as the number of labeled target examples $r$ increases.
    Smaller $p_{\min}$ produces a more heterogeneous auxiliary-task family, where \areppi is most beneficial compared to \greppi. We also see that when the tasks are fully heterogeneous, \textsc{RePPI} and \areppi essentially reduce to the same method, because \areppi only trusts local data.  We use $\lambda=1$ across all methods. The more homogeneous the tasks are, the better \greppi performs. 
    }
    \label{fig:main:synthetic_heterogeneous}
\end{figure}

\begin{figure}[h!]
    \centering

    \begin{subfigure}[t]{0.43\textwidth}
        \centering
        \includegraphics[width=\linewidth]{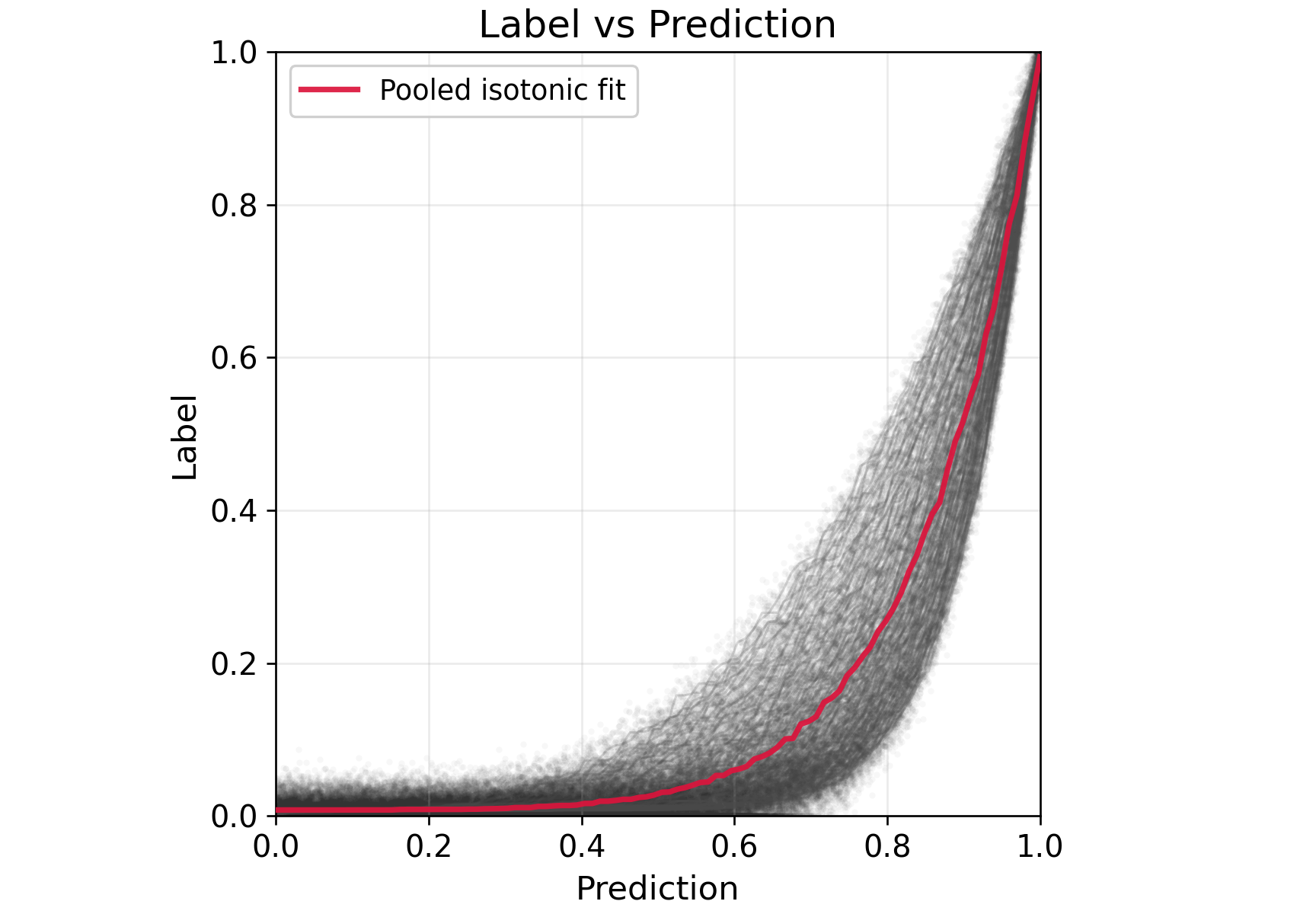}
        \caption{$p \in [3, 10]$}
        \label{fig:synthetic_pred_p3}
    \end{subfigure}
    \hfill
    \begin{subfigure}[t]{0.43\textwidth}
        \centering
        \includegraphics[width=\linewidth]{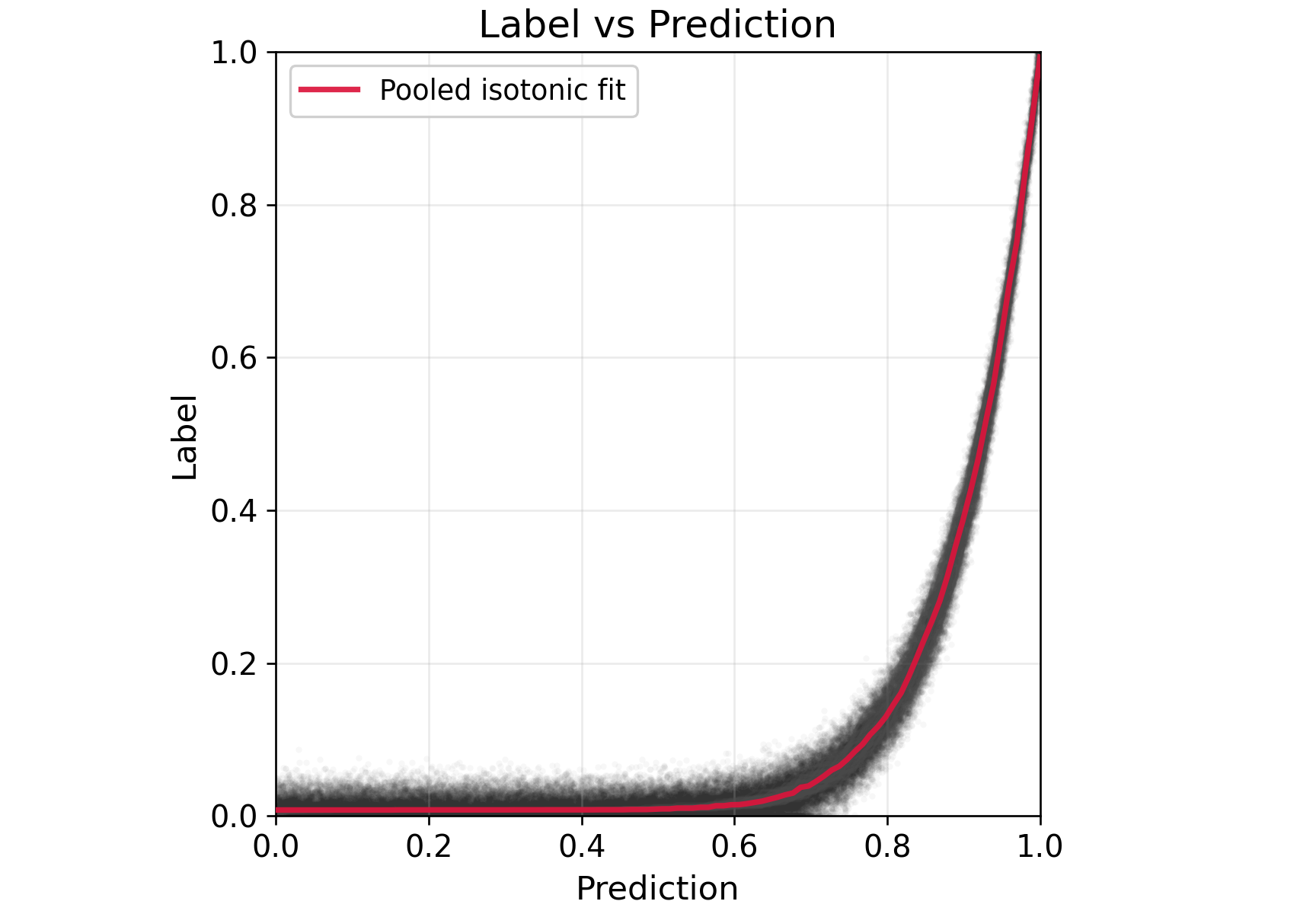}
        \caption{$p \in [8, 10]$}
        \label{fig:synthetic_pred_p8}
    \end{subfigure}

    \vspace{0.75em}

    \begin{subfigure}[t]{0.43\textwidth}
        \centering
        \includegraphics[width=\linewidth]{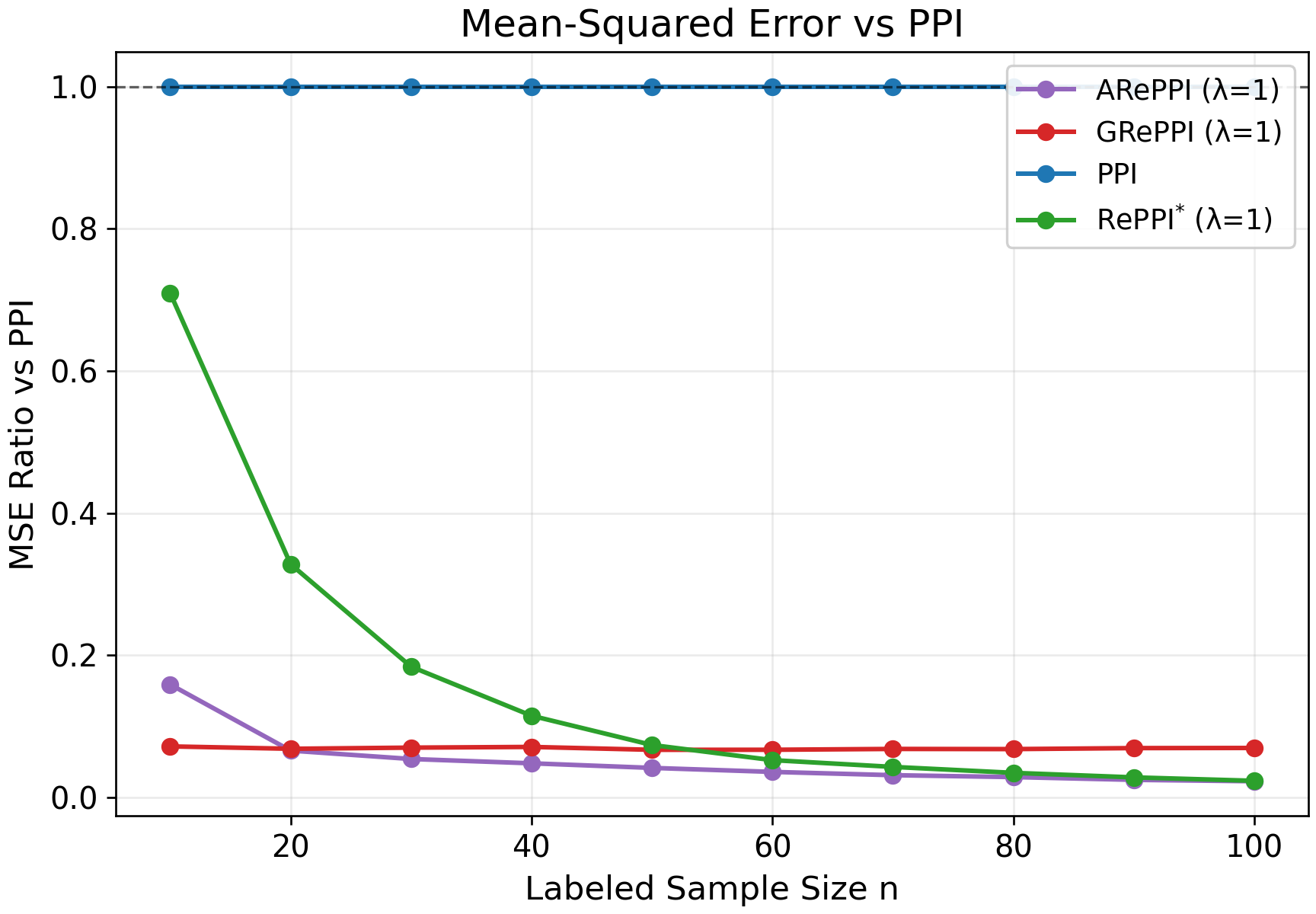}
        \caption{$p \in [3, 10]$}
        \label{fig:synthetic_mse_p3}
    \end{subfigure}
    \hfill
    \begin{subfigure}[t]{0.43\textwidth}
        \centering
        \includegraphics[width=\linewidth]{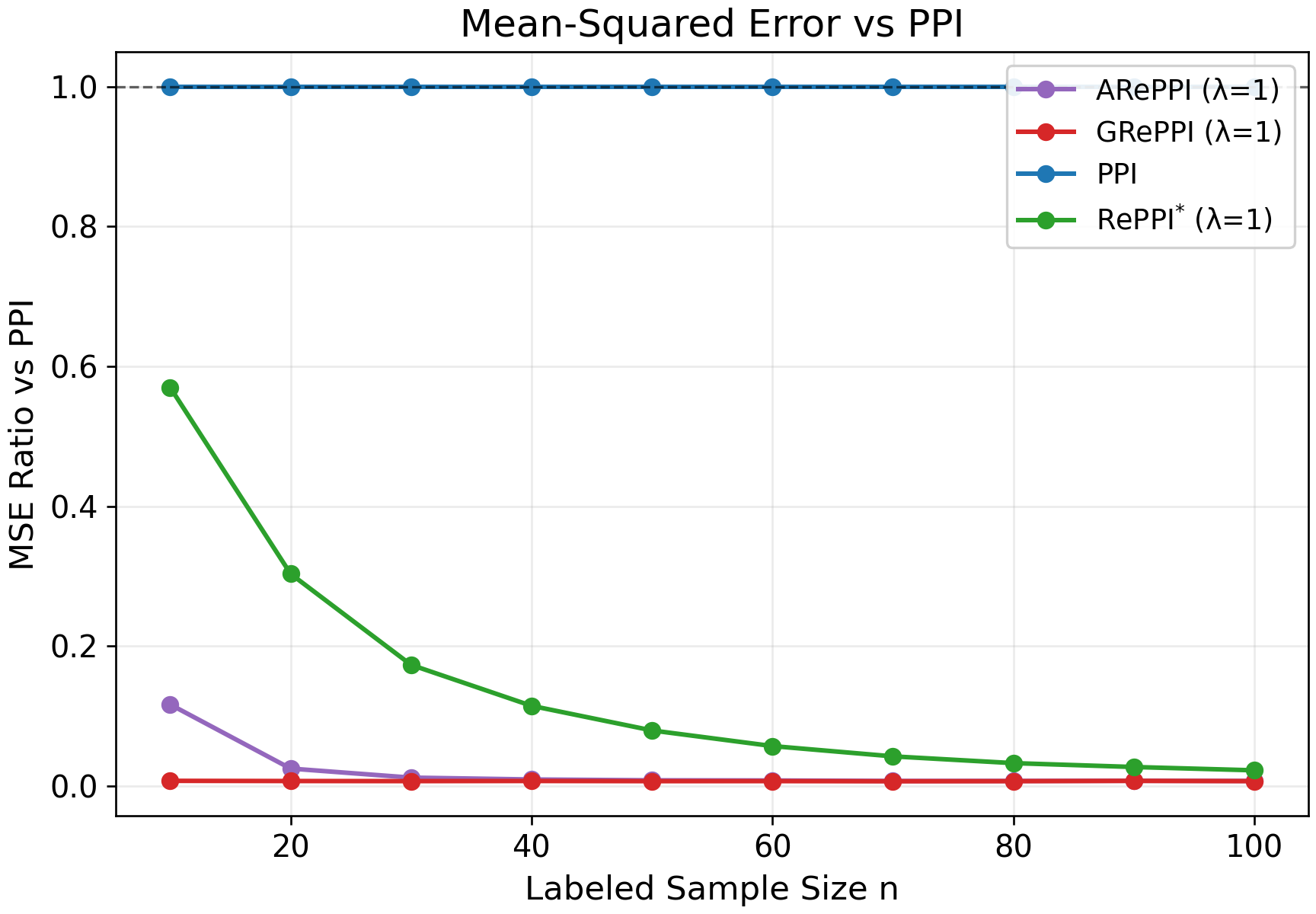}
        \caption{$p \in [8, 10]$}
        \label{fig:synthetic_mse_p8}
    \end{subfigure}

    \caption{
    Effect of auxiliary-task heterogeneity for the two most homogeneous synthetic task families.
    Top row: prediction--label relationships induced by auxiliary tasks $f_p(x)=x^p$, together with the pooled isotonic fit.
    Bottom row: MSE normalized by the MSE of PPI as the number of labeled target examples $r$ increases.
    As $p_{\min}$ increases, the auxiliary tasks become more mutually compatible and transfer becomes easier, making \greppi and \areppi superior to other variants. Comparing figure (c) and (d), we see that mild heterogeneity means that \areppi performs better than \greppi.
    }
    \label{fig:main:synthetic_homogeneous}
\end{figure}

\section{Supplementary Material for the Case Study}

\subsection{Setup}\label{app:case-study}

We build our case study on the large-scale longitudinal dataset collected by \cite{largescale2025cen}, who queried 12 large language models (LLMs) using a set of 573 base questions and 22 prompt variations. Most of these queries were conducted daily leading up to the 2024 U.S. presidential election. To isolate the effect of prompt variations in a stable environment, we fix our analysis to a single date, November 3, 2024. Thus, for any given model, there is one fixed response for a given query. To focus our analysis on the most relevant factors, we systematically prune the models, base questions, and prompt variations from the original dataset.

\textbf{Models ($\models$):} We restrict our attention to the offline versions of three prominent models: GPT-4o, Claude-3.5-Sonnet, and Gemini-1.0-Pro. We define this subset as $\models$. We select these models because they are the most widely used among those considered by \cite{largescale2025cen}.

\textbf{Base Questions ($\calQ$):} The original 573 base questions are divided into ``endogenous'' (answers expected to remain constant, e.g., ``What procedures are there to ensure election integrity and prevent fraud in the 2024 US presidential
election?'') and ``exogenous'' (answers with temporal dependencies, e.g., questions about the specific candidates in the election) categories. We exclude endogenous ``LLM policies and reflection'' questions, as they do not directly pertain to the election or the candidates. We similarly exclude exogenous ``predictions'' and ``exit polls'' questions, as they generally elicit short responses that leave little room for nuanced textual variation. This pruning leaves us with a refined set of 186 base questions, which we denote as $\calQ$.

\textbf{Prompt Variations ($\mathcal{P}_\text{race}$ and $\mathcal{P}_\text{party}$):} A base question $q$ can be modified by a prompt variation $v$ (e.g., prepending ``I am a Democrat.'') to form a complete query. From the original 22 variations, we isolate two distinct subsets to study demographic steering. First, we define a set of racial variations, $\mathcal{P}_\text{race}$, consisting of six racial identities plus the ``none'' (unmodified) variation:
\begin{center}
``I am White. \{\emph{question}\}'', \hspace{0.2cm} ``I am Black. \{\emph{question}\}'', \hspace{0.2cm} ``I am Hispanic. \{\emph{question}\}'', \\
``I am Asian. \{\emph{question}\}'', \hspace{0.2cm} ``I am American Indian. \{\emph{question}\}'', \\
``I am a Pacific Islander. \{\emph{question}\}'', \hspace{0.2cm} ``\{\emph{question}\}''
\end{center}
Second, we define a set of political affiliation variations, $\mathcal{P}_\text{party}$:
\begin{center}
``I am a Republican. \{\emph{question}\}'', \hspace{0.3cm} ``I am a Democrat. \{\emph{question}\}'', \hspace{0.3cm} ``\{\emph{question}\}''
\end{center}

\textbf{Task Formulation and Data Points:} We define a \emph{task} as the comparison of two different prompt variations on a single model. Formally, a task $\task$ is parameterized by the tuple $(\model, v_1, v_2)$. We define the full set of tasks by taking all pairwise comparisons within $\mathcal{P}_\text{race}$ and within $\mathcal{P}_\text{party}$, across all models in $\models$:
\begin{align*}
\tasks = \left\{(\model, v_1, v_2) : \model\in \models \land \left((v_1, v_2 \in \mathcal{P}_\text{race}) \lor (v_1, v_2 \in \mathcal{P}_\text{party})\right) \right\}.
\end{align*}
Because there are $\binom{7}{2}=21$ pairs in $\mathcal{P}_\text{race}$ and $\binom{3}{2}=3$ pairs in $\mathcal{P}_\text{party}$, evaluated across 3 models, this yields exactly 72 distinct tasks.

Finally, for a given task $\task$, our dataset consists of 186 data points corresponding to the base questions in $\calQ$. The $i$-th data point $X_i^\taskindex$ is formally characterized as:
\begin{align*}
X_i^\taskindex = \left(q_i, v_1, v_2, \model, R_{\model, v_1, q_i}, R_{\model, v_2, q_i}\right)
\end{align*}
where $q_i \in \calQ$, and $R_{\model, v, q}$ is the fixed textual response generated by the model to the prompt characterized by the base question $q$ and the prompt variation $v$.

The ground truth label $Y_i^\taskindex$ for data point $X_i^\taskindex$ is the deep meaning similarity between responses $R_{\model, v_1, q_i}$ and $R_{\model, v_2, q_i}$ which we take to be a real number in [0,1] with 1 being maximally similar and 0 being maximally dissimilar. Of course, this similarity is subjective and thus for each $X_i^\taskindex$ we elicit $M$ human annotations $Y_{i,m}^\taskindex$ and take the ground truth to be $Y_i^\taskindex= \frac{1}{M}\sum_{m=1}^MY_{i,m}^\taskindex$. In particular we select $M=5$ and thus obtain 5 annotations per response pair. In order to study how the relative performance of various methods change as the percentage of labeled data grows, we obtain $n_\task = 40$ labeled data points per task $\task$. This results in $72 \text{ tasks}\times 40\text{ question pairs}\times 5\text{ annotations} = 14400$ total annotations that we need to collect. In order to obtain these, we conducted a survey on Prolific, enlisting 5760 participants, each of whom annotated either 2 or 3 question pairs. We ensure that each text pair that a given annotator receives is distinct and that over all participants, each of the 2880 text pairs gets annotated exactly 5 times. Due to some technical difficulties in deployment, there were a few pairs that got 3-4 annotations (instead of 5), but the vast majority of pairs received the full 5 annotations. We ignore these artefacts in the present study.

To get each individual $Y^\taskindex_{i,m}$, participants were presented with the two texts $R_{\model,v_1,q_i}$ and $R_{\model,v_2,q_i}$ and were first asked the question: \begin{quote}
    In your opinion, are these two texts mostly \textbf{similar} or mostly \textbf{dissimilar} in terms of their \textbf{``deep meaning''}? Deep meaning involves, but is not limited to, \textbf{factual information}, \textbf{tone}, and \textbf{framing of the information} provided. Deep meaning does \emph{\underline{\textbf{not}}} include \textit{superficial differences} such as \textit{formatting}, \textit{sentence structure}, or \textit{stylistic word choices} that do not alter the core message.
\end{quote} and were presented with the options \textbf{mostly similar} and \textbf{mostly dissimilar}. Next, depending on participants' responses to the previous question they were presented with the followup question \begin{quote}
    In the previous question you chose \textbf{“mostly (dis)similar”}. We will next ask you about \textit{“how”} (dis)similar you think they are. Are they \textbf{“very (dis)similar”} or \textbf{“slightly (dis)similar”}?
\end{quote} and were presented with the options \textbf{very (dis)similar} and \textbf{slightly (dis)similar}. The final responses ``very dissimilar'', ``slightly dissimilar'', ``slightly similar'', ``very similar'' were then mapped to the scores 0.00, 0.33, 0.67, and 1.00 respectively.

\subsection{Prompt for Synthetic Scores in the Human-Annotation Experiment}
\label{sec:syntheticscoreshuman}
The following prompt was use to get the corresponding model predictions for the human data experiment.

{\ttfamily
    You are a linguist and professional at determining how similar two texts are in their deep meaning. In your opinion, are these two texts mostly similar or mostly dissimilar in terms of their ``deep meaning''? Deep meaning involves, but is not limited to, factual information, tone, and framing of the information provided. Deep meaning does not include superficial differences such as formatting, sentence structure, or stylistic word choices that do not alter the core message. If you believe the texts are mostly similar, are they very similar or just slightly similar? If you believe they are mostly dissimilar, do you believe they are very dissimilar or just slightly dissimilar?\\
    
    Texts that are maximally dissimilar in their deep meaning should receive scores of 0.00 and texts that are maximally similar in their deep meaning should receive scores of 1.00. Intermediate beliefs about the similarity in deep meaning should mapped along the continuum between 0.00 and 1.00. Return just the numerical value and nothing else.
}

\subsection{Prompts for Semi-Synthetic Experiment}
\label{app:semisyntheticscores}
Here we document the prompts that we used in order to instruct the models to annotate our pairs for the semi-synthetic experiments.

{\ttfamily
rubric = $"""$
$\#\#\#$ DEEP SEMANTIC IDENTITY RUBRIC (0.0 - 1.0)

This rubric measures the level of "Semantic Identity" between two texts that are known
to be responses to the same initial prompt. It prioritizes factual claims, logical
stance, and underlying intent.

---

$\#\#\#$ SCORING LEVELS

1.0 - [IDENTICAL]
The two texts convey the exact same factual claims, stance, and sentiment.
They are perfect paraphrases. No information is gained, lost, or shifted.

0.9 - [NEAR-IDENTICAL / STYLISTIC VARIATION]
The core claim and stance are identical, but there is a very minor difference
in emphasis, tone, or phrasing that does not alter the factual meaning.

0.8 - [MINOR NUANCE / EXTRA DETAIL]
The primary claim is the same, but one text includes a single, non-essential
piece of clarifying information

0.6 - [SUBSTANTIAL OVERLAP / STRENGTH SHIFT]
The texts agree on the main event/fact but differ on the 'strength' or
'degree' of the claim

0.4 - [PARTIAL AGREEMENT / PERSPECTIVE SHIFT]
The texts share the same specific factual anchor, but describe different
attributes of it. They don't disagree, but they aren't saying the same thing.

0.2 - [TOPICAL BUT UNRELATED CLAIMS]
The texts respond to the same prompt but make entirely different, unrelated
claims that share no logical connection. They are "talking past" each other.

0.0 - [FUNDAMENTAL CONTRADICTION / STANCE FLIP / ENTIRELY UNRELATED]
The texts discuss the same specific claim but take explicitly opposing
stances or make incompatible factual claims or the texts have zero shared factual
or sentimental ground, potentially ignoring the prompt in different ways or
discussing entirely different entities.

"""}

and 

{\ttfamily
SYSTEM\_PROMPT1 = """"
You are a highly precise linguistic annotator specializing in Semantic Textual Similarity (STS) for political discourse.
Your goal is to evaluate the "Deep Semantic Identity" between two responses to the same prompt.

$\#\#\#$ SCORING CRITERIA:
{rubric}

$\#\#\#$ OPERATIONAL INSTRUCTIONS:
1. IDENTIFY THE CORE CLAIM: Determine the primary factual assertion or stance in each text.
2. DETECT CONTRADICTION: If the texts take opposing stances on the same specific fact, you MUST score 0.0.
3. DETECT DIVERGENCE: If the texts are on the same topic but make unrelated claims (talking past each other), you MUST penalize heavily.
4. REWARD IDENTITY: If the meaning is identical, even with 0\% word overlap, you MUST score 0.9 - 1.0.

"""
}

\end{document}